%% file: main.tex
\documentclass[11pt, letterpaper, logo, onecolumn, copyright, numbering]{acc}

\usepackage{url}

\input{pkgs}

\title{\sysrlt: Scaling Long-Horizon LLM Agents via Indexed Experience Memory}

\author[1]{Zhenting Wang}
\author[1]{Huancheng Chen}
\author[1]{Jiayun Wang}
  \author[1]{Wei Wei}

  \affil[1]{Center for Advanced
   AI, Accenture}

  \date{\today}

\input{contents/abstract}

\begin{document}

\maketitle
\def\Snospace~{Section }
\def\sectionautorefname{\Snospace}
\def\subsectionautorefname{\Snospace}
\def\subsubsectionautorefname{\Snospace}
\def\chapterautorefname{\Snospace}

\input{contents/introduction.tex}

\input{contents/related}
\input{contents/method.tex}

\input{contents/evaluation}

\input{contents/conclusion.tex}
\newpage

\raggedright
\bibliography{reference}
\bibliographystyle{iclr2026_conference}

\newpage
\input{contents/appendix}

\end{document}

%% file: pkgs.tex
\usepackage{natbib}

\newcommand{\sys}{\mbox{\textsc{Memex}}\xspace}
\newcommand{\sysrl}{\mbox{\textsc{MemexRL}}\xspace}
\newcommand{\sysrlt}{\mbox{\textsc{Memex(RL)}}\xspace}

\usepackage{amsthm}

\usepackage{times}
\newtheorem{theorem}{Theorem}
\newtheorem{proposition}[theorem]{Proposition}

\usepackage{xcolor}
\definecolor{darkblue}{RGB}{32, 64, 129}
\definecolor{darkgreen}{RGB}{0, 110, 85}
\definecolor{darkred}{RGB}{153, 0, 0}
\definecolor{graytext}{gray}{0.45}

\definecolor{evaluatorcolor}{RGB}{255,230,230}    %
\definecolor{evaluatorframe}{RGB}{180,50,50}      %
\definecolor{taskgencolor}{RGB}{230,255,230}      %
\definecolor{taskgenframe}{RGB}{50,180,50}        %
\definecolor{executioncolor}{RGB}{230,230,255}    %
\definecolor{executionframe}{RGB}{50,50,180}      %
\definecolor{lightgray}{RGB}{240,240,240}
\definecolor{darkgray}{RGB}{80,80,80}

\usepackage{pifont}
\usepackage{longtable}

\newtheorem{definition}{Definition}

\usepackage{makecell}

\usepackage{amsmath}
\usepackage{pifont}

\usepackage{booktabs}
\usepackage{multirow}
\usepackage{graphicx}
\usepackage{float}
\usepackage{caption}
\usepackage{subcaption}
\usepackage{xspace}

\usepackage[utf8]{inputenc}
\usepackage[T1]{fontenc}
\usepackage[dvipsnames,table]{xcolor}
\usepackage[most]{tcolorbox}
\usepackage{enumitem}
\usepackage{listings}

\definecolor{darkgray}{rgb}{0.3, 0.3, 0.3}
\definecolor{lightgray}{rgb}{0.95, 0.95, 0.95}
\definecolor{codegray}{rgb}{0.98, 0.98, 0.98}

\newtcolorbox{promptbox}[1]{
    colback=lightgray,
    colframe=darkgray,
    colbacktitle=darkgray,
    coltitle=white,
    boxrule=2pt,
    arc=0mm,
    left=10pt,
    right=10pt,
    top=10pt,
    bottom=10pt,
    fonttitle=\bfseries\large,
    title={#1},
    attach boxed title to top left={yshift=-2mm}
}

\usepackage{wrapfig}

\usepackage{algorithm}
\usepackage{algorithmicx}
\usepackage[noend]{algpseudocode}

\definecolor{deepred}{rgb}{0.631,0.102,0.102}
\definecolor{skyblue}{HTML}{126da2}

\definecolor{accpurple}{HTML}{A100FF}
\hypersetup{
     colorlinks = true,
     breaklinks = true,
     linkcolor = accpurple,
     urlcolor = accpurple,
     citecolor = accpurple
     }

\definecolor{orange}{rgb}{1,0.5,0}

\algnewcommand{\LineComment}[1]{\State \(\triangleright\) #1}

\algnewcommand\algorithmicinput{\textbf{Input:}}
\algnewcommand\algorithmicoutput{\textbf{Output:}}
\algnewcommand\Input{\State \algorithmicinput\ }
\algnewcommand\Output{\State \algorithmicoutput\ }

%% file: contents/abstract.tex
\begin{abstract}
Large language model (LLM) agents are fundamentally bottlenecked by finite context windows on long-horizon tasks. As trajectories grow, retaining tool outputs and intermediate reasoning in-context quickly becomes infeasible: the working context becomes prohibitively long, eventually exceeds the context budget, and makes distant evidence harder to use even when it is still present. Existing solutions typically shorten context through truncation or running summaries, but these methods are fundamentally lossy because they compress or discard past evidence itself. We introduce \sys, an indexed experience memory mechanism that instead compresses context without discarding evidence. \sys maintains a compact working context consisting of concise structured summaries and stable indices, while storing full-fidelity underlying interactions in an external experience database under those indices. The agent can then decide when to dereference an index and recover the exact past evidence needed for the current subgoal. We optimize both write and read behaviors with our reinforcement learning framework \sysrl, using reward shaping tailored to indexed memory usage under a context budget, so the agent learns what to summarize, what to archive, how to index it, and when to retrieve it. This yields a substantially less lossy form of long-horizon memory than summary-only approaches. We further provide a theoretical analysis showing the potential of the Memex loop to preserve decision quality with bounded dereferencing while keeping effective in-context computation bounded as history grows. Empirically, on challenging long-horizon tasks, \sys agent trained with \sysrl improves task success while using a significantly smaller working context.
\end{abstract}

%% file: contents/introduction.tex
\section{Introduction}
\label{sec:intro}

Large language model (LLM) agents are increasingly deployed as general-purpose problem solvers that plan, call tools, and interact with users over extended periods~\citep{openai2025gpt5,google2025Gemini3ProModelCard,anthropic2025claude4,yang2025qwen3,chen2025minimax,team2026kimi,zeng2026glm}. Unlike single-shot prompting or short multi-turn chats, long-horizon agents are asked to execute workflows spanning dozens to hundreds of steps and tool calls, often producing large volumes of intermediate observations, tool outputs, and reasoning traces along the way, such as searching and cross-referencing scientific literature, exploring configuration spaces for code and infrastructure, orchestrating multi-API business processes, or iteratively refining complex analyses for a user~\citep{zhou2025mem1,yan2025memory,wu2025resum,sun2025scaling}. In these settings, success depends not only on local reasoning quality, but also on whether the agent can preserve and later reuse information that first appeared many steps earlier, such as a constraint mentioned at the beginning of a conversation, a failure mode surfaced by a tool, or an API response that only becomes decisive much later.

This creates a fundamental bottleneck for long-horizon agents. Although modern LLMs support increasingly large context windows, those windows remain finite, while agent trajectories naturally keep growing as observations, tool outputs, and intermediate reasoning are appended over time. A straightforward strategy is to retain as much of this history as possible in-context, but this quickly becomes inefficient or infeasible. Prompts become prohibitively long, eventually exceed the available context budget, and make distant evidence harder to use even when it is still present. Existing systems mostly address this pressure through static context engineering, such as large rolling prompts, heuristic summarization, or related memory heuristics. These approaches can reduce active working context, but they usually do so by either truncating substantial portions of past interactions or compressing them into lossy summaries that are difficult to faithfully recover later~\citep{yan2025memory,wu2025resum,sun2025scaling}.
A natural alternative is to log everything into an external memory and retrieve past content by semantic similarity~\citep{corley2005measuring} when needed. However, in long-horizon tool use, this design is often brittle. When memory consists of a large pool of noisy, near-duplicate fragments, retrieval becomes ambiguous, and the model must repeatedly re-parse loosely structured history. More fundamentally, similarity-based retrieval does not specify how the agent should organize its own experience. It does not determine which intermediate results deserve stable references, which branches are dead ends, or how artifacts should be named so that later access is precise rather than fuzzy. As a result, many current systems still rely on hand-designed templates and heuristics for memory construction and retrieval. This leaves a gap between how humans manage long-term work, keeping a small set of active concepts while preserving exact notes and references externally, and how LLM agents manage theirs.

\input{figtex/intro}

In this work, we argue that long-horizon LLM agents need a memory mechanism that \emph{compresses context without discarding evidence}. We instantiate this idea with \textbf{\sys}, whose core component is \emph{Indexed Experience Memory}. As illustrated in Figure~1, \sys replaces a long tool-use trajectory in the working context with a compact indexed summary, while archiving the full underlying artifacts in an external key--value experience store under stable indices. Later, when a specific past result becomes relevant, the agent can explicitly dereference an index to recover the exact archived content and re-inject it into the working context. In this way, \sys explicitly separates a compact in-context \emph{working context} from a full-fidelity external \emph{experience archive}. At each step, the agent writes a concise, structured \emph{indexed summary} into its working context, recording actionable progress together with stable references. In parallel, the full underlying artifacts needed for faithful later reuse, such as tool outputs, logs, code snippets, and other fine-grained evidence, are archived in the external store under those indices. Subsequent decisions operate primarily over this short indexed working context, and the agent brings back raw evidence only through explicit index dereferencing when it becomes relevant to the current subgoal. This makes memory access precise and auditable because an index points to a concrete archived artifact rather than to an approximate semantic match.
Crucially, \sys is not a hand-crafted memory heuristic. We treat memory operations, including writing indexed summaries, archiving artifacts, and dereferencing indices, as first-class actions in the same decision space as environment tools. This introduces a distinctive long-horizon credit assignment problem. A well-timed compression or a well-designed index may only pay off many steps later by enabling precise evidence recovery, avoiding redundant tool calls, and preventing context overflow. In contrast, a locally plausible but poorly structured summary can silently derail downstream reasoning. To learn these behaviors, we introduce \textbf{\sysrl}, a reinforcement learning framework that optimizes both the \emph{write} policy, including what to summarize, what to archive, how to index, and when to compress, and the \emph{read} policy, including when and what to dereference, under a context budget. \sysrl combines reward shaping tailored to indexed memory usage with a compression-adaptive training procedure that preserves learning signal for delayed memory decisions across long trajectories. We also expose context status to the agent through a soft triggering mechanism, turning compression timing into a learnable skill rather than a fixed system rule.
Conceptually, \sys follows a simple principle: keep the active reasoning state small, but do not throw evidence away. Instead of forcing the agent to repeatedly reason over an ever-growing prompt or over lossy summaries, \sys maintains a compact control state in-context while preserving exact past interactions off-context for later reuse. This design is closer to how humans manage long, tool-heavy work, where external artifacts such as notes, file names, and bookmarks serve as stable access routes to detailed evidence without requiring everything to remain in working memory~\citep{andy1998extended}. It also echoes index-based accounts of episodic memory, in which compact retrieval cues enable later reactivation of richer underlying traces~\citep{teyler1986hippocampal}. \sys operationalizes these principles in an agent setting. The indexed summary serves as an actionable working state, while stable indices provide explicit retrieval cues into a full-fidelity experience database.

To clarify why this design is promising, we also provide a theoretical analysis of the Memex loop. We show its potential to support two desirable properties at the same time: preserving decision quality through bounded explicit dereferencing, and keeping the agent's effective in-context computation bounded as the full message history grows. This analysis is not meant to claim that \sysrl always learns such summaries in practice. Rather, it characterizes the regime in which compact indexed summaries together with bounded retrieval are sufficient to support accurate decisions with bounded working context.
We then study whether this theoretical potential can be realized in practice. To do so, we evaluate \sys on challenging long-horizon tasks that require agents to interleave planning and tool calls over dozens to hundreds of steps while revisiting fine-grained evidence long after it first appears. In this regime, \sys improves task success while operating with a small active working context under tight context budgets.
Overall, this paper makes four contributions. First, we introduce Indexed Experience Memory, a memory interface that pairs a compact in-context indexed summary with a full-fidelity external experience archive under stable references, enabling precise explicit dereferencing of past evidence. Second, we propose \sysrl, a reinforcement learning framework that trains memory write and read behaviors with reward shaping and compression-adaptive trajectory processing. Third, we provide a theoretical analysis showing the potential of the Memex loop to preserve decision quality with bounded dereferencing while keeping effective in-context computation bounded as history grows. Fourth, we provide an empirical study on long-horizon tasks showing that learned indexed experience memory improves task success under tight context budgets while maintaining a small active working context.

%% file: figtex/intro.tex
\begin{figure}[t]
    \centering
    \includegraphics[width=0.85\textwidth]{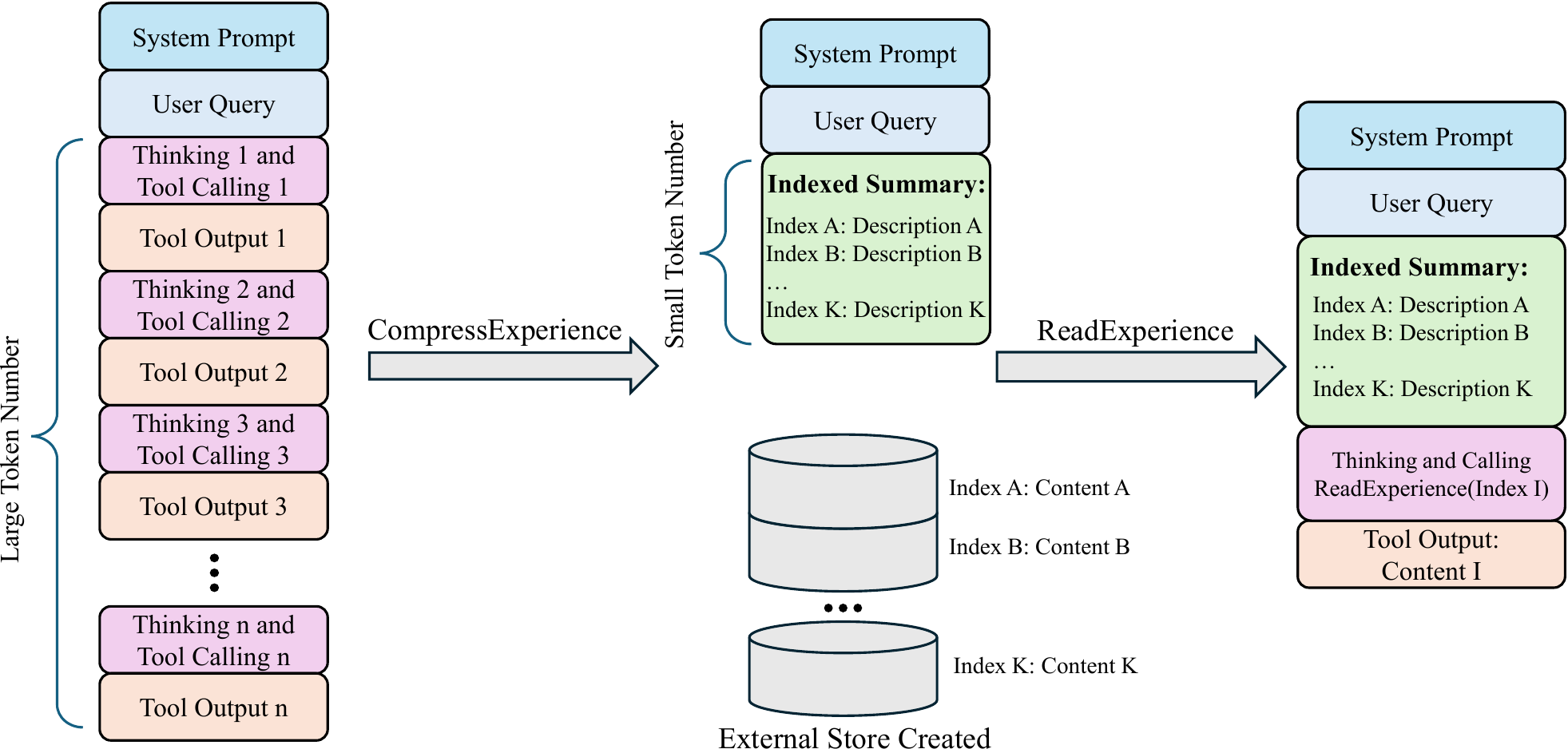}
    \caption{ Memex agent loop overview. CompressExperience replaces a long tool-use trajectory in the context with a compact indexed summary, while storing detailed contents in an external key–value store. Later, ReadExperience(index) dereferences an index to retrieve the exact content and re-inject it into the context, enabling long-horizon execution under a small context window.
    }\label{fig:intro}
    \vspace{-0.3cm}
\end{figure}

%% file: contents/related.tex
\section{Related Work}
\label{sec:related}

\textbf{LLM Memory.}
LLM memory has emerged as an important research direction for enabling models and agents to store, organize, and retrieve information beyond their immediate context~\citep{Hu-Survey2025}. Early frameworks such as MemGPT \citep{Packer-Arxiv2023} and MemoryBank \citep{Zhong-AAAI2024}, which organize dialogue history, accumulate experience, and self-evolve through continuous multi-turn updates, have inspired a line of subsequent works \citep{Rezazadeh-ICLR2025, Chhikara-Arxiv2025, Kang-Arxiv2025, Tan-ACL2025,Rasmussen-Arxiv2025} that focus on efficiently storing memory using various structured representations by dynamic mechanism. Long-term memory in LLM agents can be broadly categorized into factual memory and experiential memory.
Factual memory functions as a declarative knowledge base that supports consistency \citep{Zhou-Arxiv2023, Xu-NIPS2025,Long-Arxiv2025}, coherence \citep{Liu-Arxiv2023, Chen-Arxiv2025}, and logs environment states \citep{Wang-NAACL2024, Jimenez-NIPS2024}, whereas experiential memory captures knowledge accumulated from historical trajectories \citep{Zhao-AAAI2024, Zhou-Arxiv2025, zhang2025memgen} and interaction patterns \citep{Shinn-NIPS2023, Wang-AWM-ICML2024, Ouyang-Arxiv2025}, enabling continual learning and self-evolution. Another line of work integrates information directly into a model’s internal components—such as hidden states \citep{Wang-m+-ICML2025} and key–value caches \citep{Qian-memorag-WC2025}, rather than relying on external memory databases, in pursuit of more efficient information compression. 
Despite these advances, a central challenge in LLM memory remains how to preserve salient information while enabling efficient compression and retrieval.

\textbf{Long-Horizon LLM Agents.}
Although LLMs have finite context windows, long-horizon agent trajectories naturally accumulate observations, actions, and intermediate reasoning over many steps, causing the prompt to grow until it becomes inefficient or exceeds the available context budget. This ever-growing prompt not only increases inference cost and memory usage \citep{yao2023react,kwon2023efficient,zheng-sglang-NIPS2024,liu2026klong}, but also makes it harder for the model to effectively use distant evidence even when that evidence remains somewhere in the retained history \citep{An-ICLR2025,wu-longmemeval-ICLR2025}. To address this challenge, recent work has explored learned memory management for long-horizon settings. Early RL-based approaches such as MEM1 \citep{zhou2025mem1}, MemAgent \citep{yu2025memagent}, Memory-R1 \citep{yan2025memory}, and Mem-$\alpha$ \citep{wang2025mem} train models to compress, organize, or update historical information into compact memories, primarily in question-answering or long-context settings. More recent work extends learned summarization to interactive agentic reinforcement learning and long-horizon search. SUPO \citep{lu2025scaling} and ReSum \citep{wu2025resum} incorporate memory summarization into multi-turn tool-use and search pipelines, while FoldGRPO \citep{sun2025scaling} and AgentFold \citep{ye2025agentfold} further study structured context folding and multi-scale historical summaries. However, most existing approaches still center on lossy summary compression or generic memory organization: they reduce the active context, but typically do not preserve a precise, full-fidelity archive of past evidence that can be deterministically revisited later. In contrast, Memex formulates memory as indexed experience management: it keeps a compact in-context working state, while archiving full-fidelity artifacts off-context under reusable stable indices. This design allows the agent to not only decide when to compress, but also what to archive, how to index it, and when to explicitly dereference exact past evidence for downstream reasoning, making memory access substantially less lossy than summary-only approaches.

%% file: contents/method.tex
\section{\sys Agent}
\label{sec:method}

\subsection{Overview}
LLM agents must operate under a fixed context window. As interactions accumulate, keeping the
entire trajectory in-context quickly becomes infeasible, while lossy compression often removes precisely the evidence
needed later (e.g., logs, thinking processes, and intermediate tool outputs). The key component of \sys agent is the \textbf{Indexed Experience Memory},
which keeps a compact in-context state while preserving full-fidelity artifacts in an external experience store under
stable indices. The agent periodically rewrites its working context into a short, actionable \emph{indexed summary} and
accesses archived evidence only through explicit index dereferencing when needed.

\subsection{Indexed Experience Memory}
\label{sec:method_interface}

We maintain the agent context window \(\mathbf{M}\) and an external experience store
\(\mathcal{D}:\mathsf{index}\mapsto\mathsf{content}\), a key--value database accessed by explicit index dereferencing.
Throughout this paper, the \emph{working context} refers to the dynamic portion of \(\mathbf{M}\) beyond the fixed system prompt and task instruction. Namely, \(\mathbf{M} = [m_0, u, \mathbf{M}_{\mathrm{work}}]\), where \(\mathbf{M}_{\mathrm{work}}\) denotes the working context. Depending on the stage of interaction, \(\mathbf{M}_{\mathrm{work}}\) may include the indexed summary, the agent's intermediate reasoning, tool calls, tool outputs, and any retrieved experience blocks injected back into context.
We next define the indexed summary \(\sigma\) and the indexed experience memory operations.

\begin{definition}[Indexed Summary]
\label{def:indexedsummary}
Given an external experience store $\mathcal{D}:\mathsf{index}\mapsto \mathsf{content}$,
an \emph{Indexed Experience Summary} $\sigma$ is an in-context state $\sigma=(s,\mathcal{I})$, where
$s$ is a compact, actionable progress state (e.g., verified information, and plans), and
$\mathcal{I}$ is a finite set of pairs
\(
\mathcal{I} \triangleq \{(\mathsf{index}, \mathsf{description})\},
\)
where $\mathsf{index}\in \mathrm{dom}(\mathcal{D})$ is a stable index and $\mathsf{description}$ is a summarized descriptor.
\end{definition}

\begin{definition}[Indexed Experience Memory]
\label{def:memory}
Let $\mathbf{M}$ denote the agent's context window, initialized as $\mathbf{M}=[m_0,u]$
(system prompt and task instruction), and let $\mathcal{D}$ be an external experience store.
An \emph{Indexed Experience Memory} is a pair $(\mathsf{IndexedSummary}, \mathcal{D})$ with:

\textbf{(i) In-context indexed summary.} $\mathsf{IndexedSummary}$ is a structured state kept in-context
that records the actionable progress state,
and provides an \emph{index map} that binds semantic summarized descriptions to indices in $\mathcal{D}$, enabling
explicit reference to archived evidence. See \autoref{def:indexedsummary} for detailed definition.

\textbf{(ii) External Store.} $\mathcal{D}: \mathsf{index}\mapsto \mathsf{content}$ maps \emph{stable indices}
to archived content blocks (e.g., tool outputs, traces, code snippets). The agent can access archived
evidence only by dereferencing indices via $\textsc{ReadExperience}(i)$, which returns $\mathcal{D}[i]$.

\textbf{(iii) Compression Operation.} $\textsc{CompressExperience}(\mathsf{IndexedSummary}, \mathsf{MemoryBlocks})$
performs archival by writing each $(\mathsf{index},\mathsf{content})\in\mathsf{MemoryBlocks}$ into
$\mathcal{D}$, and then rewrites the working context to $\mathbf{M}\leftarrow[m_0,u,\mathsf{IndexedSummary}]$.
Intuitively, it transforms an ever-growing working context into a small indexed summary.

\textbf{(iv) Read Operation.} $\textsc{ReadExperience}(\mathsf{index})$ returns the archived block by dereferencing
$o\leftarrow \mathcal{D}[\mathsf{index}]$, and appends it to the working context as a new message
$\mathbf{M}\leftarrow \mathbf{M}\oplus[o]$.
\end{definition}

For compression operation, note that the content in each memory block supports two modes:
(a)~\emph{explicit authoring}, where the model writes the content directly (e.g., reorganized notes or summarized findings); and
(b)~\emph{anchor-based extraction}, where the model specifies three short text anchors (\texttt{start\_anchor}, \texttt{mid\_anchor}, \texttt{end\_anchor}) that uniquely identify a span within the current conversation. The system locates the matching span and archives it verbatim, with the \texttt{mid\_anchor} serving as a verification checkpoint to prevent false matches.
This dual-mode design gives the model flexibility to paraphrase for storage efficiency or preserve exact content (e.g., precise object IDs, code snippets, test outputs) without generating redundant tokens.

\paragraph{Examples for working context compression.}
We first illustrate how Indexed Experience Memory's compression operation rewrites an ever-growing context into a compact,
pointer-heavy state. Here we show a concrete before/after compression example:
multiple rounds of reasoning and tool outputs are archived into $\mathcal{D}$ under stable indices, while the
working context is rewritten to $\mathbf{M}=[m_0,u,\mathsf{IndexedSummary}]$. The \(t_i\) and \(o_i\) are the thinking with tool call request and the tool output at step \(i\), respectively.

\input{examples/context}

\input{alg/alg}

\paragraph{Agent Loop with Indexed Experience Memory.}
\autoref{alg:exp_memory_main_msgs_only} specifies the full execution loop.
Initialization sets $\mathbf{M}\gets[m_0,u]$, $\mathcal{D}\gets\emptyset$, and $answer\gets\varnothing$
(lines 4--6), where the two anchors $m_0$ (system prompt) and $u$ (task instruction) are never compressed.
At each step $t$ (line 7), the system first appends \textsc{ContextStatus}$(\mathbf{M},\tau)$ (line 8),
a deterministic message computed from the current context window $\mathbf{M}$ that reports its working context token usages and the threshold $\tau$ (e.g., ``[Context Status: working context tokens=6932, threshold=8000]''). The agent then emits thinking $z_t$ and a tool call $c_t$
conditioned on $\mathbf{M}$ (line 9), and appends them to the working context (line 10).
If $c_t=\textsc{CompressExperience}(\mathsf{IndexedSummary},\mathsf{MemoryBlocks})$, the agent archives each
$(\mathsf{index},\mathsf{content})\in\mathsf{MemoryBlocks}$ into the external store via
$\mathcal{D}[\mathsf{index}]\leftarrow \mathsf{content}$ (lines 12--13), and rewrites the working context to
$\mathbf{M}\leftarrow[m_0,u,\mathsf{IndexedSummary}]$ (line 14).
If $c_t=\textsc{ReadExperience}(\mathsf{index})$, the agent dereferences the store via
$o_t\leftarrow \mathcal{D}[\mathsf{index}]$ (line 16) and appends the retrieved block as a new message
$\mathbf{M}\leftarrow \mathbf{M}\oplus[o_t]$ (line 17).
If $c_t=\textsc{Finish}(y)$, the agent sets $answer\gets y$ (line 19) and returns (line 20).
Otherwise, it executes other tools (line 22) and appends the observation to $\mathbf{M}$ (line 23).
If no \textsc{Finish} occurs within $T_{\max}$ steps, the procedure terminates (line 24).

\subsection{\sysrl: Learning Indexed Experience Memory Enabled Agent with RL}
\label{sec:method_rl}

The agent loop in \autoref{alg:exp_memory_main_msgs_only} treats memory operations
(\textsc{CompressExperience}, \textsc{ReadExperience}) as tool calls in the same action space as
environment tools and \textsc{Finish}.
A compression step must decide \emph{what to archive}, \emph{how to index it}, and \emph{what to keep in-context}:
archiving the right evidence under stable indices and writing an actionable index-map summary may only pay off many
steps later by enabling precise dereferencing, preventing redundant tool calls, and avoiding context overflow.
Conversely, missing or ambiguous indices can silently derail future reasoning even when the summary appears locally
plausible.
These decisions are inherently difficult to enforce with prompt rules alone, since the utility of an
index entry is revealed only when a downstream step needs to retrieve exactly that artifact.
We therefore learn memory behaviors jointly with task-solving behaviors using reinforcement learning.
We optimize the agent policy with \textbf{\sysrl}, a GRPO-style update specialized to tool-using agents with
Indexed Experience Memory, encouraging trajectories that solve the task while using compression and dereferencing
efficiently and correctly.

\paragraph{Reward design for \sysrl.}
  We define an episode-level return that combines task success with three memory-efficiency penalties,
  each normalized to $[0,1]$:
  \begin{equation}
  R \;=\; R_{\text{task}} - P_{\text{context}} - P_{\text{redundancy}} - P_{\text{format}}.
  \label{eq:memgrpo_reward}
  \end{equation}

  \emph{Context overflow penalty.}
  Let $C_t$ denote the working context size (in tokens) at step $t$, excluding steps where compression was triggered.
  The penalty accumulates overflow tokens beyond threshold $\tau$ across all $T$ steps:
  $P_{\text{context}} = \min\!\left(1,\;\frac{\sum_{t=1}^{T}\max(0,\,C_t-\tau)}{\tau \cdot T}\right)$.
  Normalizing by the maximum reasonable overflow $\tau \cdot T$ bounds the penalty in $[0,1]$.
  This encourages the agent to proactively invoke \textsc{CompressExperience}
  before context grows excessively, rather than waiting until forced truncation degrades performance.

  \emph{Redundant tool call penalty.}
  Let $N_{\text{redundant}}$ count tool calls with identical (tool name, arguments) signatures that were previously
  executed, when no state-modifying operation has occurred since the last identical call. The penalty is:
  $P_{\text{redundancy}} = \frac{N_{\text{redundant}}}{N_{\text{tool\_call}}}$,
  where $N_{\text{tool\_call}}$ is the total number of non-memory tool calls in the trajectory.
  This penalty discourages repetitive exploration patterns---such as
  viewing the same file region multiple times without edits---and encourages the agent to utilize
  \textsc{ReadExperience} to recall previously observed information instead of re-executing identical queries.

  \emph{Format error penalty.}
  Let $N_{\text{malformed}}$ count malformed tool calls detected in the agent's raw output, including:
  (1) tool call tag mismatches (e.g., \texttt{<tool\_call>} without closing \texttt{</tool\_call>}),
  (2) invalid JSON within tool call tags, and
  (3) missing required fields such as \texttt{name} or \texttt{arguments}.
  The penalty is:
  \(
  P_{\text{format}} = \frac{N_{\text{malformed}}}{N_{\text{tool\_call}}}\),
  where $N_{\text{tool\_call}}$ is the number of steps where the agent attempted to invoke a tool.
  This provides a direct learning signal for generating syntactically correct tool invocations,
  which is particularly important for base models without prior tool-use fine-tuning.

\paragraph{Segmented Trajectory Processing.}
When compression occurs during a trajectory, we segment the trajectory at compression boundaries and process each segment as an independent training sample while preserving credit assignment to the final outcome. This design is particularly natural for autoregressive models, whose token probabilities are always conditioned on the current prefix. Once compression is triggered, the prefix for all subsequent generation changes: instead of continuing from the full pre-compression history, the model now conditions on a compressed summary together with the new interactions that follow. In this sense, compression does not merely shorten the context; it changes the conditioning prefix that defines the downstream policy.
Concretely, suppose an agent compresses $k$ times during a trajectory, producing $k+1$ segments $\{S_0, S_1, \ldots, S_k\}$. Each segment $S_i$ maintains its own context. $S_0$ contains the full pre-compression history, while each subsequent segment $S_i$ for $i > 0$ contains the compressed summary from the previous segment together with new interactions, i.e.,
\(
S_i = [\texttt{system}, \texttt{task}, \texttt{summary}_{i-1}, z_{i1}, c_{i1}, o_{i1}, z_{i2}, c_{i2}, o_{i2}, \ldots].
\)
Here, $z_{ij}$ denotes the model's intermediate reasoning or textual generation at the $j$-th interaction step of segment $S_i$, $c_{ij}$ denotes the corresponding tool call or action command, and $o_{ij}$ denotes the resulting observation returned by the environment or tool. Thus, each segment consists of a sequence of reasoning--action--observation tuples appended after the inherited summary context.
During training, we flatten all segments across the batch, allowing each segment to be tokenized and optimized independently under its own context window. Crucially, all segments originating from the same trajectory share the identical terminal reward $R$, which preserves credit assignment to earlier compression decisions through group-relative advantage estimation in GRPO. Thus, the model can learn whether a compression-induced prefix helped or harmed the eventual outcome, and the write policy receives gradient signal from the shared terminal reward even though its effect may only appear in later segments.

\paragraph{Automatic Compression Triggering.}
Unlike prior work that enforces compression automatically when context length exceeds a threshold~\citep{lu2025scaling}, we adopt a soft triggering approach that automatically monitors context status and prompts the agent to compress voluntarily. Specifically, at each step, we append a context status indicator to the observation: \texttt{[Context Status: working=$w$, total=$t$, threshold=$L$]}. When the working context $w$ approaches or exceeds $L$, additional warnings are automatically injected (e.g., \texttt{``working > threshold''}). This design transforms context management from a system-enforced constraint into a learnable skill. The agent must learn to (1) recognize when compression is beneficial, (2) decide the optimal timing for compression based on task semantics rather than arbitrary token counts, and (3) generate high-quality summaries that preserve task-relevant information. Through reinforcement learning with context overflow penalties, the
agent is incentivized to compress proactively while maintaining task performance. This approach offers greater flexibility than hard thresholds---for instance, the agent may defer compression if it anticipates task completion within a few steps, or compress earlier at natural semantic boundaries.

\paragraph{Overall learning algorithm.} For each training instance, we sample a group of $G$ rollouts
$\{y^{(g)}\}_{g=1}^{G}$ from the current policy, where each rollout is a sequence of model tokens that includes
both the agent’s thinking text and a tool call at each step (including memory tools and \textsc{Finish}).
We execute the resulting tool calls to obtain a scalar return $R^{(g)}$ using \autoref{eq:memgrpo_reward}, and
compute a group-relative, normalized advantage
\begin{equation}
A^{(g)} \;=\; \frac{R^{(g)}-\mathrm{mean}\!\big(\{R^{(h)}\}_{h=1}^{G}\big)}
{\mathrm{std}\!\big(\{R^{(h)}\}_{h=1}^{G}\big)+\epsilon}.
\label{eq:memgrpo_adv}
\end{equation}
Let $\pi_{\theta}$ be the current policy and $\pi_{\theta_{\text{old}}}$ be the sampling policy for the group.
For each rollout, we form the token-level ratio
$r^{(g)}_t=\pi_{\theta}(a^{(g)}_t\mid s^{(g)}_t)/\pi_{\theta_{\text{old}}}(a^{(g)}_t\mid s^{(g)}_t)$ and apply a
PPO-style clipped surrogate objective, optionally regularized by a KL penalty to a reference policy $\pi_{\text{ref}}$:
\begin{equation}
\max_{\theta}\;\frac{1}{G}\sum_{g=1}^{G}\sum_{t}
\min\!\Big(r^{(g)}_t A^{(g)},\;\mathrm{clip}(r^{(g)}_t,1-\eta,1+\eta)\,A^{(g)}\Big)
\;-\;\beta\,\mathrm{KL}\!\big(\pi_{\theta}\,\|\,\pi_{\text{ref}}\big).
\label{eq:memgrpo_obj}
\end{equation}
Because $R^{(g)}$ explicitly scores context-budget violations, redundant tool use, and oversized archival blocks,
the group-relative advantage in \autoref{eq:memgrpo_adv} encourages the policy to discover memory behaviors that
improve long-horizon returns: compressing into pointer-heavy index maps that remain usable later, archiving
high-value artifacts under stable indices, and selectively retrieving evidence instead of re-running tools.
In practice, we warm-start \sysrl from supervised demonstrations of well-formed $\sigma$ and block sets
$\mathsf{Blocks}$, then refine with RL under the same message-only interface.

%% file: examples/context.tex
\begin{tcolorbox}[
    enhanced,
    breakable,
    colback=lightgray,
    colframe=darkgray,
    colbacktitle=darkgray,
    coltitle=white,
    boxrule=2pt,
    arc=0mm,
    left=10pt,
    right=10pt,
    top=10pt,
    bottom=10pt,
    fonttitle=\bfseries\large,
    fontupper=\scriptsize,
    fontlower=\scriptsize,
    title={Example for Compression with Indexed Experience Memory (Before/After)},
    attach boxed title to top left={yshift=-2mm},
    halign=flush left,
    before upper=\raggedright
]
\label{ex:indexed_summary_compare}

\textbf{Context Before compression:}\\
\[
\mathbf{M}=[m_0,u,t_1,o_1,t_2,o_2,t_3,o_3,t_4,o_4,t_5,o_5,t_6,o_6]
\]
\begin{description}[leftmargin=1.5em, itemsep=2pt, topsep=2pt]
    \item[$m_0$] System prompt: tool-usage rules and context management constraints. 
    \item[$u$] User task: ``SI.\_collect\_factor\_and\_dimension() cannot properly detect that exponent is dimensionless'' with a repro using \texttt{exp} and a failing ValueError......
    \item[$t_1$] Reasoning: orient in repo to find units code; Tool call: list /testbed tree (2 levels). (approx. 350 tokens)
    \item[$o_1$] Repo snapshot: /testbed structure with key paths under \texttt{sympy/physics/units}. (approx. 600 tokens)

    \item[$t_2$] Reasoning: locate the core method; Tool call: search \texttt{\_collect\_factor\_and\_dimension} in \texttt{sympy/physics/units}. (approx. 280 tokens)
    \item[$o_2$] Search hits: 28 matches, including \texttt{sympy/physics/units/unitsystem.py} (9 matches). (approx. 240 tokens)

    \item[$t_3$] Reasoning: jump to exact definition; Tool call: search for \texttt{def \_collect\_factor\_and\_dimension} in \texttt{unitsystem.py}. (approx. 220 tokens)
    \item[$o_3$] Exact location: match at line 152 in \texttt{unitsystem.py}. (approx. 120 tokens)

    \item[$t_4$] Reasoning: inspect dimension handling in context; Tool call: view \texttt{unitsystem.py} lines 100--160. (approx. 420 tokens)
    \item[$o_4$] Code excerpt: \texttt{get\_dimensional\_expr} and start of \texttt{\_collect\_factor\_and\_dimension}. (approx. 900 tokens)

    \item[$t_5$] Reasoning: inspect full Function branch; Tool call: view \texttt{unitsystem.py} lines 140--206. (approx. 420 tokens)
    \item[$o_5$] Code excerpt: full \texttt{\_collect\_factor\_and\_dimension} showing Function branch return inconsistency. (approx. 1100 tokens)

    \item[$t_6$] Reasoning: preserve reproducible failing case; Tool call: create \texttt{reproduce\_issue.py}. (approx. 280 tokens)
    \item[$o_6$] Reproducer content mirroring the GitHub issue (exp on dimensionless expr). (approx. 450 tokens)
\end{description}

\tcblower

\textbf{Context After compression:}\\
\[
\mathbf{M}=[m_0,u,\mathsf{IndexedSummary}]
\]
\textbf{IndexedSummary (approx. 300 tokens):}\\
\textbf{CurrentStatus:} Next, inspect \texttt{sympy/physics/units/dimensions.py} for \texttt{Dimension} and \texttt{DimensionSystem} to ensure proper construction of \texttt{Dimension(1)}. Then patch \texttt{unitsystem.py} Function branch to return \texttt{(expr.func of factors, Dimension(1))} when args are dimensionless; otherwise, decide correct behavior (likely raise error or compute dimension expression). Finally, run \texttt{reproduce\_issue.py} before and after patch to confirm fix.\\[4pt]

\textbf{IndexMap:}
\begin{description}[leftmargin=1.5em, itemsep=2pt, topsep=2pt]
    \item[ctx\_repo\_snapshot\_001:] /testbed directory snapshot (2-level listing) for orientation; use to avoid re-running repo view.
    \item[ctx\_units\_search\_001:] search hits for "\_collect\_factor\_and\_dimension" across units module; shows key files and counts.
    \item[ctx\_units\_method\_loc\_001:] precise location of def \_collect\_factor\_and\_dimension in unitsystem.py (line 152) to jump directly.
    \item[ctx\_units\_code\_excerpt\_001:] excerpt (lines 100-160) of unitsystem.py showing get\_dimensional\_expr and beginning of \_collect\_factor\_and\_dimension; crucial for understanding dimension handling and function branch logic.
    \item[ctx\_units\_code\_excerpt\_002:] excerpt (lines 140-206) of unitsystem.py covering the full \_collect\_factor\_and\_dimension logic including Function branch (lines 192-196); identifies inconsistency in return type for Function.
    \item[ctx\_repro\_script\_001:] reproduce\_issue.py content mirroring the GitHub issue to validate bug and later confirm fix.
\end{description}
\end{tcolorbox}

%% file: alg/alg.tex
\begin{algorithm}[t]
\caption{\sys Agent Loop}
\label{alg:exp_memory_main_msgs_only}
\begin{algorithmic}[1]
\Input Task instruction $u$, maximum steps $T_{\max}$, context threshold $\tau$
\Output Submitted answer $answer$
    \Statex \textbf{Tool set:} $\mathcal{T}=\{\textsc{CompressExperience}(\cdot),\textsc{ReadExperience}(\cdot),\textsc{Finish}(\cdot),\textsc{OtherTool}(\cdot)\}$

\Function{Execute}{$u, T_{\max}, \tau$}
    \State $\mathbf{M}\gets[m_0,u]$ \Comment{System + task; never compressed}
    \State $\mathcal{D}\gets\emptyset$ \Comment{Memory Pool}
    \State $answer\gets\varnothing$
    \For{$t=1$ \textbf{to} $T_{\max}$}
        \State $\mathbf{M}\gets \mathbf{M}\oplus[\textsc{ContextStatus}(\mathbf{M},\tau)]$
        \State $(z_t, c_t)\gets \pi_{\text{agent}}(\mathbf{M})$
        \Comment{$z_t$: thinking, $c_t$: tool call (incl.\ Finish/Compress/Read)}
        \State $\mathbf{M}\gets \mathbf{M}\oplus[z_t, c_t]$

        \If{$c_t = \textsc{CompressExperience}(\mathsf{IndexedSummary},\mathsf{MemoryBlocks})$}
            \For{$(\mathsf{index}, \mathsf{content})$ \textbf{in} $\mathsf{MemoryBlocks}$}
                \State $\mathcal{D}[\mathsf{index}] \leftarrow \mathsf{content}$
            \EndFor
            \State $\mathbf{M}\gets [m_0,u,\mathsf{IndexedSummary}]$
        \ElsIf{$c_t = \textsc{ReadExperience}(\mathsf{index})$}
            \State $o_t \gets \mathcal{D}[\mathsf{index}]$
            \State $\mathbf{M}\gets \mathbf{M}\oplus[o_t]$ \Comment{Retrieved content as a new message}
        \ElsIf{$c_t = \textsc{Finish}(y)$}
            \State $answer\gets y$
            \State \Return $answer$
        \Else
            \State $o_t \gets \textsc{ExecuteTool}(c_t)$
            \State $\mathbf{M}\gets \mathbf{M}\oplus[o_t]$ \Comment{Environment/tool output as a new message}
        \EndIf
    \EndFor
    \State \Return \textsc{MaxStepReached}
\EndFunction
\end{algorithmic}
\end{algorithm}

%% file: contents/evaluation.tex
\section{Theoretical Analysis of Why the \sys Loop Scales to Long-Horizon Agents}
\label{sec:theory}

This section examines the theoretical potential of the \sys agent loop. In particular, we ask whether Indexed Experience Memory can, in principle, support two desirable properties at the same time: preserving optimal decision quality without conditioning on the full message history, and keeping the agent's effective in-context computation bounded as the message history grows. Our goal is not to prove that \sysrl always learns such summaries in practice, but to characterize the regime in which indexed summaries and bounded dereferencing are sufficient to support accurate decisions with bounded working context.

\paragraph{Property 1: Preserving decision quality with bounded dereferencing.}
Let \(M_t^{\mathrm{full}}\) denote the full message history up to step \(t\), and let \((\sigma_t, D_t)\) be the \sys state induced by \(M_t^{\mathrm{full}}\), where \(\sigma_t\) is the in-context indexed summary and \(D_t : \mathsf{index} \mapsto \mathsf{content}\) is the external experience store.

We say that the indexed summary is sufficient for decision making if it does not need to contain all past evidence itself, but can instead point to the small subset of archived evidence that matters for the current decision.

\begin{definition}[Decision-sufficient indexed summary]
We say that \(\sigma_t\) is \emph{\(B\)-bounded decision-sufficient} if there exist an index selector \(g\) and a decision function \(\mu\) such that \(|g(\sigma_t)| \leq B\) and \(\pi^*(\cdot \mid M_t^{\mathrm{full}}) = \mu\!\left(\cdot \mid \sigma_t, \{D_t[i]\}_{i \in g(\sigma_t)}\right)\), where \(\pi^*\) is an optimal policy that conditions on the full message history \(M_t^{\mathrm{full}}\).
\end{definition}

This definition formalizes the desired \sys property. Although the full message history may be arbitrarily long, the optimal action can still be recovered from a compact indexed summary together with at most \(B\) explicitly dereferenced blocks from the archive.

\begin{proposition}[\sys can match a full-context optimal policy]
Let \(J(\pi)\) denote the expected return of policy \(\pi\). Assume that \(\sigma_t\) is \(B\)-bounded decision-sufficient for every step \(t\). Then there exists a \sys policy \(\pi_{\mathrm{IEM}}\) that conditions only on \(\sigma_t\) and uses at most \(B\) calls to \textsc{ReadExperience}\((\cdot)\) per step such that \(J(\pi_{\mathrm{IEM}}) = J(\pi^*)\).
\end{proposition}

\noindent
\textit{Interpretation.}
This result says that if the indexed summary plus a bounded number of dereferences capture all decision-relevant evidence, then \sys loses no decision quality relative to an agent that always conditions on the full message history.

\noindent
\textit{Proof sketch.}
At each step, the policy first uses \(g(\sigma_t)\) to select at most \(B\) indices, dereferences the corresponding blocks from \(D_t\), and then applies \(\mu\) to \((\sigma_t, \{D_t[i]\}_{i \in g(\sigma_t)})\). By definition, this reproduces the same action distribution as the optimal full-context policy \(\pi^*\) at every step, and therefore achieves the same expected return. Detailed proof can be found in \autoref{sec:theory_proofs}.

\paragraph{Property 2: Keeping working context bounded as trajectory grows.}
The second property concerns efficiency. \sys is useful only if it keeps the working context small even when the full history becomes very long. Let \(C_t^{\mathrm{full}} \triangleq |M_t^{\mathrm{full}}|\) denote the token length of the full message history. Under \sys, the model does not need to keep all of \(M_t^{\mathrm{full}}\) in its working context. Instead, the working context at step \(t\) consists only of the indexed summary \(\sigma_t\) together with the dereferenced blocks used at that step. We therefore define the working-context length as \(C_t^{\mathrm{work}} \triangleq |\sigma_t| + \sum_{i \in I_t} |D_t[i]|\), where \(I_t\) is the set of indices read at step \(t\).

\begin{proposition}[\sys keeps working context bounded]
\label{th:context_bound}
Assume that for all \(t\), the summary length satisfies \(|\sigma_t| \leq \tau_\sigma\), the number of dereferenced indices satisfies \(|I_t| \leq B\), and each retrieved block satisfies \(|D_t[i]| \leq L\). Then \(C_t^{\mathrm{work}} \leq \tau_\sigma + BL \triangleq C_{\mathrm{work}}^{\max}\). Consequently, the compression ratio \(\rho_t \triangleq \frac{C_t^{\mathrm{full}}}{C_t^{\mathrm{work}}}\) satisfies \(\rho_t \geq \frac{C_t^{\mathrm{full}}}{C_{\mathrm{work}}^{\max}}\). In particular, as the full message history length \(C_t^{\mathrm{full}}\) grows, the working context remains bounded while the compression ratio grows without bound.
\end{proposition}

\noindent
\textit{Interpretation.}
This result formalizes the central efficiency advantage of \sys. Even as the full message history keeps growing, the agent only needs to maintain a bounded working context, provided that the summary remains compact and only a bounded number of archived blocks are explicitly dereferenced at each step. Since the system prompt \(m_0\) and task instruction \(u\) are fixed, bounded working context also implies bounded effective in-context computation. Detailed proof can be found in \autoref{sec:theory_proofs}.

\section{Empirical Results for \sysrl}
\label{sec:results}

The previous section establishes the theoretical potential of the \sys loop: in principle, indexed experience memory can preserve decision quality while keeping the working context bounded. We now turn to empirical evaluation to examine whether this potential can be realized in practice by our training method \sysrl.

\subsection{Experiments Setup}
\label{sec:setup}

\noindent
\textbf{LLM Models.}
We use Qwen3-30B-A3B-Thinking-2507~\citep{yang2025qwen3} in our experiments. It is a Mixture-of-Experts (MoE) model with about 30B total parameters, 128 experts per layer with top-8 routing, and approximately 3B active parameters per token. It has strong tool-use understanding and instruction-following capabilities, which allow it to understand the memory tools provided in the prompt and produce useful \textsc{CompressExperience} and \textsc{ReadExperience} usage patterns before RL training. 

\input{figtex/results}

\input{figtex/eval}

\noindent
\textbf{Environments.} The environment used in our experiments is a \emph{modified and harder version} of ALFWorld~\citep{shridhar2020alfworld}. We make the following modifications to make it more suitable for evaluating long-horizon memory agents: (1) \emph{Hidden admissible commands}: the default ALFWorld observation includes a list of all valid actions with exact object IDs, which trivializes navigation. We remove this from the observation, requiring the agent to discover valid actions through exploration. (2) \emph{Hidden initial observation}: the room description containing all location IDs is stripped from the initial observation. The agent must explicitly call the look'' action to observe its surroundings, placing location IDs in the conversation history rather than the fixed system context. (3) \emph{Limited look}: the look'' action is restricted to once per episode. After the first use, the location IDs can only be recovered through \textsc{ReadExperience}, creating an explicit dependency on the memory retrieval mechanism. (4) \emph{Summary truncation}: the summary produced by \textsc{CompressExperience} is truncated to 300 tokens, forcing the agent to store detailed information (e.g., object IDs) in structured \texttt{db\_blocks} and retrieve them via \textsc{ReadExperience} rather than embedding everything in the summary. The training set contains 3553 tasks. Detailed prompts used can be found in \autoref{sec:prompts}.

\noindent
\textbf{Training Details.} Our implementation is based on the open-sourced LLM RL framework Slime~\citep{slime_github}. We use INT4 quantization for inference rollout and quantization-aware training (QAT) for the backward pass, where weights are quantized to INT4 for forward computation but gradients are accumulated in BF16. To mitigate the inference-training mismatch introduced by multi-turn agent interactions, we adopt truncated importance sampling (TIS) with a clip ratio of 2.0 for loss calculation and a token-in-token-out agent loop that avoids re-tokenization between turns. 
Training uses the Adam optimizer with learning rate $5 \times 10^{-6}$, weight decay 0.1, and a KL divergence penalty ($\lambda_{\text{KL}}=0.001$) against the frozen reference model. The context window size and the context penalty threshold used in training are 32K and 8K, respectively.
The batch size is 32 (i.e., 32 prompts for each rollout step).
The group size for GRPO is 8 (i.e., 8 rollout samples per prompt).

\subsection{Effectiveness}

We evaluate the effectiveness of \sysrl by comparing the \sys agent with and without RL training on the modified ALFWorld benchmark. The training dynamics already show that the proposed learning framework is effective. As shown in \autoref{fig:alfworld_success}, the rollout task success rate increases from approximately 20\% to over 90\% during training. At the same time, \autoref{fig:alfworld_penalty} shows that the total penalty improves from around $-0.4$ to approximately $-0.1$. These trends indicate that the agent not only becomes substantially better at completing tasks, but also learns to manage its working context more effectively through the memory actions provided by \sys.

The final evaluation results in \autoref{fig:eval} further confirm this improvement. \sysrl increases the task success rate from 24.22\% to 85.61\%, while reducing the peak working context length from 16934.46 to 9634.47 tokens. This corresponds to more than a $3.5\times$ improvement in task success together with an approximately 43\% reduction in peak working context. These results show that the gain from \sysrl is not simply due to more aggressive compression. Instead, RL teaches the agent how to use indexed experience memory in a way that better supports downstream reasoning and decision making under a constrained context budget. Note that the context penalty
threshold used in training is 8000. The results show that the peak working context length after \sysrl training is close to the threshold.

\autoref{fig:memory_usage} provides additional insight into the learned memory behavior. After training, the mean number of \texttt{CompressExperience} calls per episode decreases from about 6.5 to about 3, while the mean number of \texttt{ReadExperience} calls increases from about 1 to around 6--7. This behavioral shift is important. It suggests that \sysrl does not merely encourage the agent to compress history more frequently. Rather, the learned policy compresses more selectively and increasingly relies on explicit retrieval from the external experience store when previously observed evidence becomes relevant again. In other words, RL shifts the agent from repeatedly rewriting context toward building a reusable indexed memory that can be dereferenced precisely when needed. This is exactly the intended operating mode of \sys, and it explains why the trained agent achieves substantially better long-horizon performance with a significantly smaller working context.

\input{figtex/eval_behaviors}

%% file: figtex/results.tex
\begin{figure}[t]
    \centering
    \includegraphics[width=0.75\textwidth]{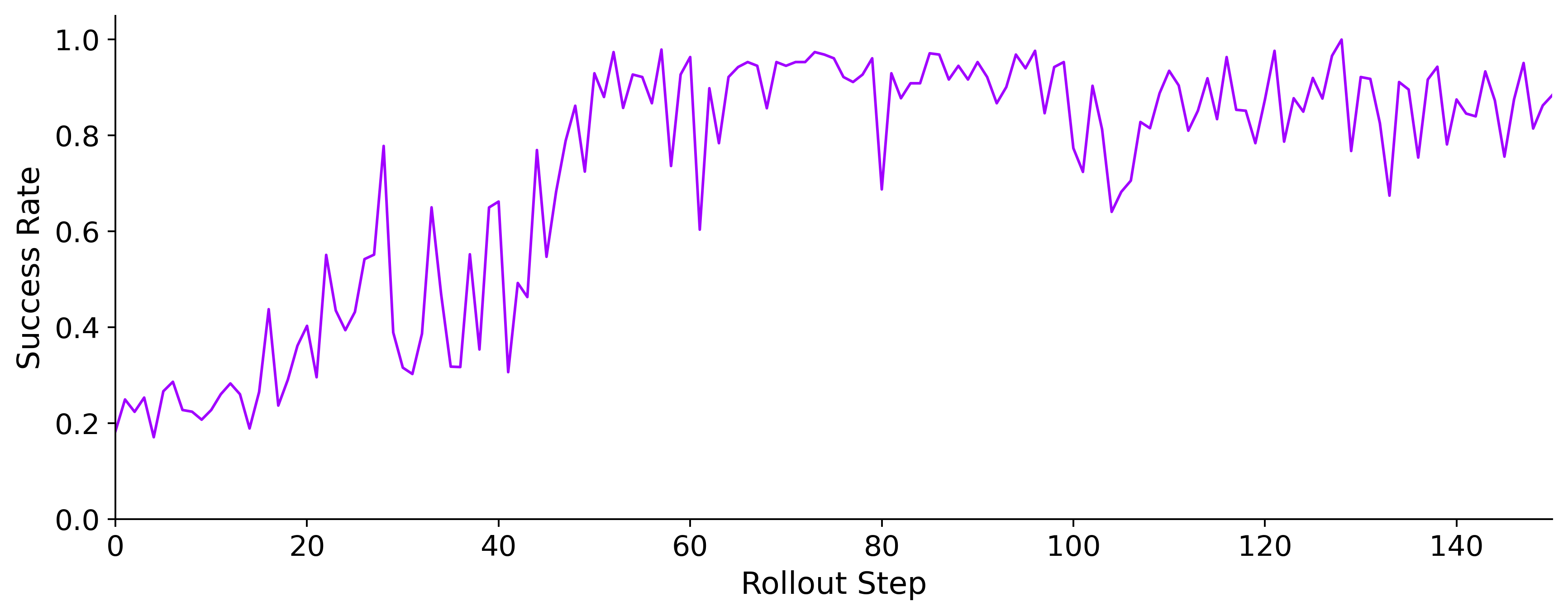}
    \caption{Task success rates for rollouts during training. The agent's task success rate improves from approximately 20\% to over 90\%, demonstrating that \sysrl training effectively teaches the model to solve tasks using \sys agent loop.}\label{fig:alfworld_success}
\end{figure}

\begin{figure}[t]
    \centering
    \includegraphics[width=0.75\textwidth]{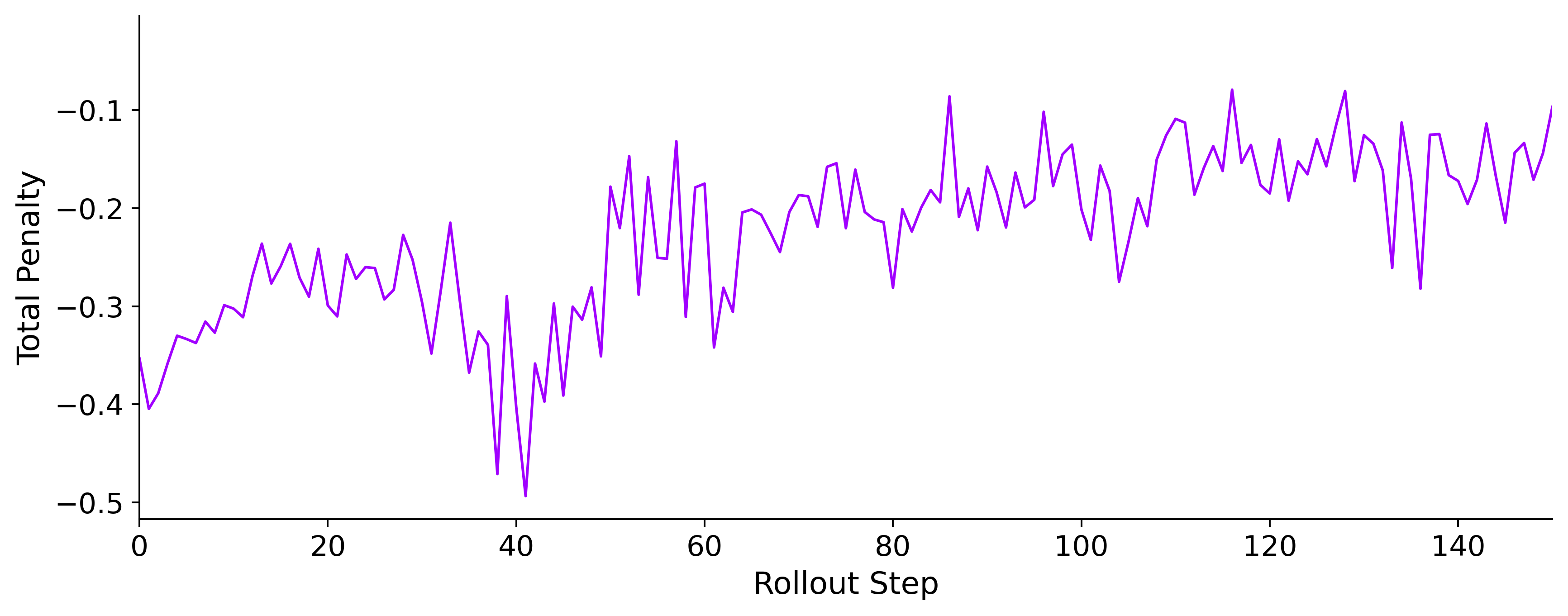}
    \caption{Total penalty for rollouts during training. The penalty decreases in magnitude from $-0.4$ to approximately $-0.1$, showing that the agent learns better task execution and strategically squeezing the peak working context length using \textsc{CompressExperience} and \textsc{ReadExperience}.
    }\label{fig:alfworld_penalty}
\end{figure}

%% file: figtex/eval.tex
\begin{figure}[t]
    \centering
    \begin{subfigure}[b]{0.38\textwidth}
        \centering
        \includegraphics[width=\textwidth]{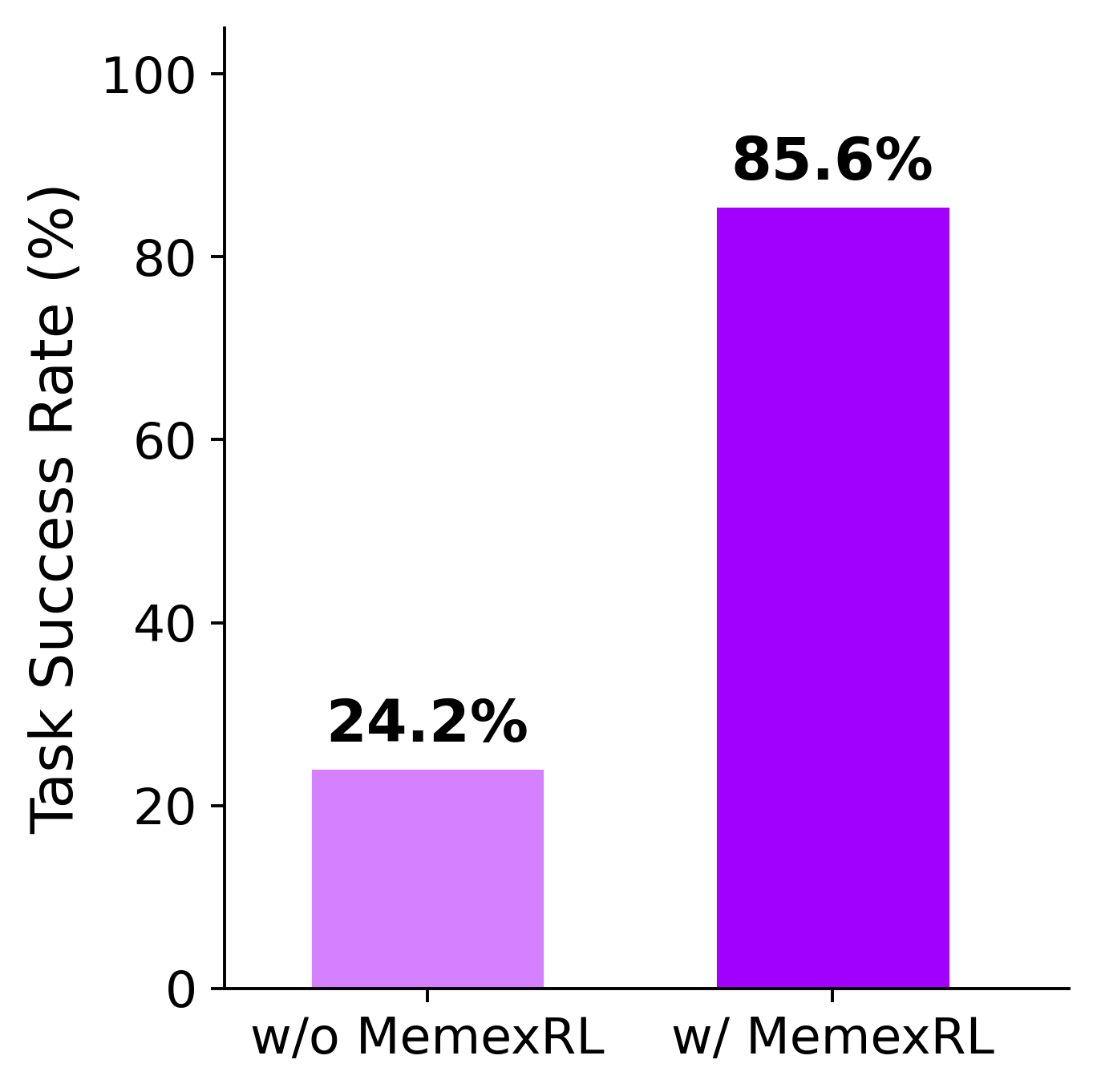}
        \caption{Task success rate.}
        \label{fig:bar_success}
    \end{subfigure}
    \hfill
    \begin{subfigure}[b]{0.38\textwidth}
        \centering
        \includegraphics[width=\textwidth]{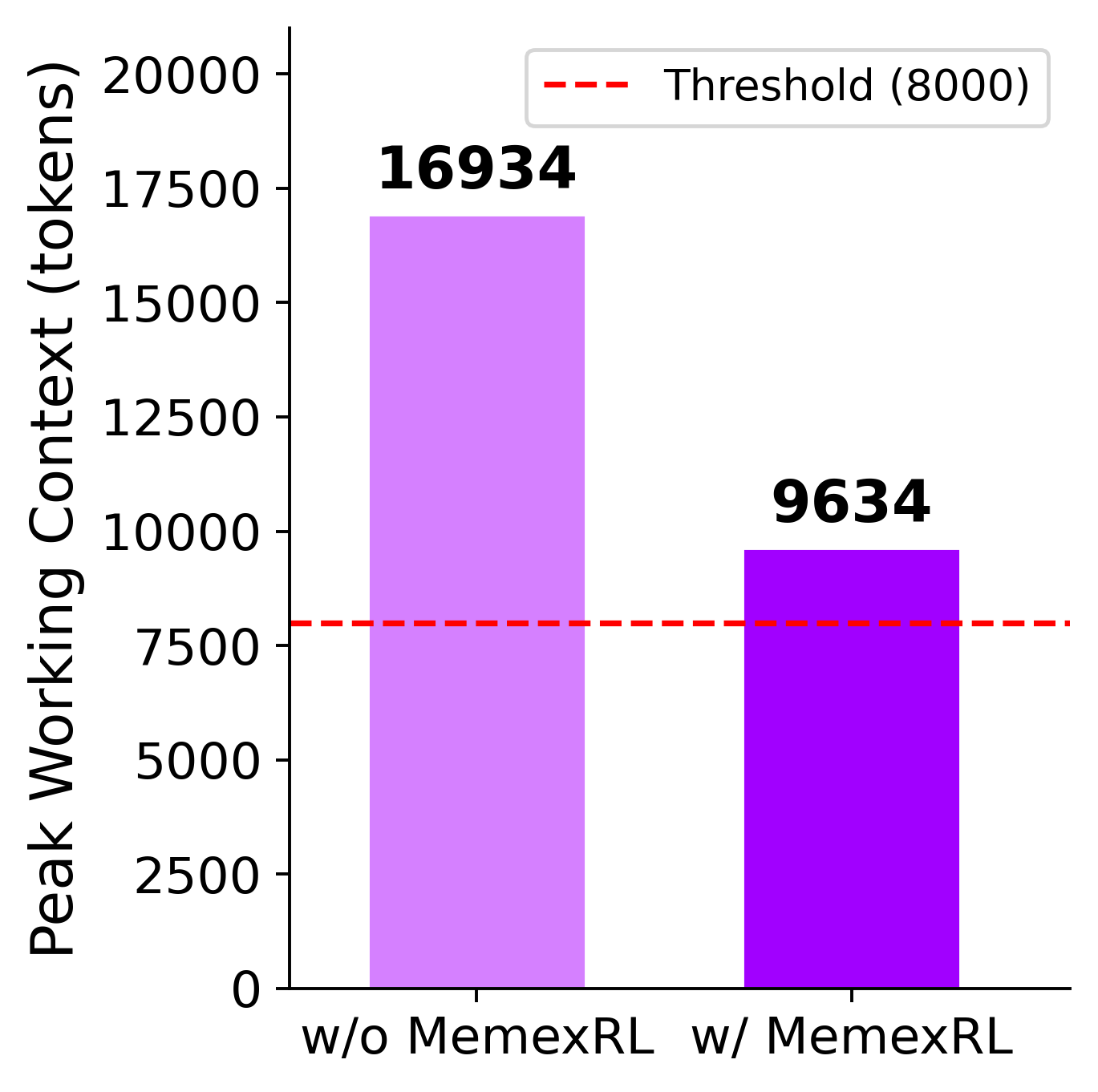}
        \caption{Peak working context length.}
        \label{fig:bar_context}
    \end{subfigure}
    \caption{Effectiveness of \sysrl. (a) Task success rate improves from 24.2\% to 85.6\%. (b) Peak working context length reduces from 16,934 to 9,634 tokens, approaching the penalty threshold of 8,000 tokens.}
    \label{fig:eval}
\end{figure}

%% file: figtex/eval_behaviors.tex
\begin{figure}[t]
    \centering
    \begin{subfigure}[b]{0.48\textwidth}
        \centering
        \includegraphics[width=\textwidth]{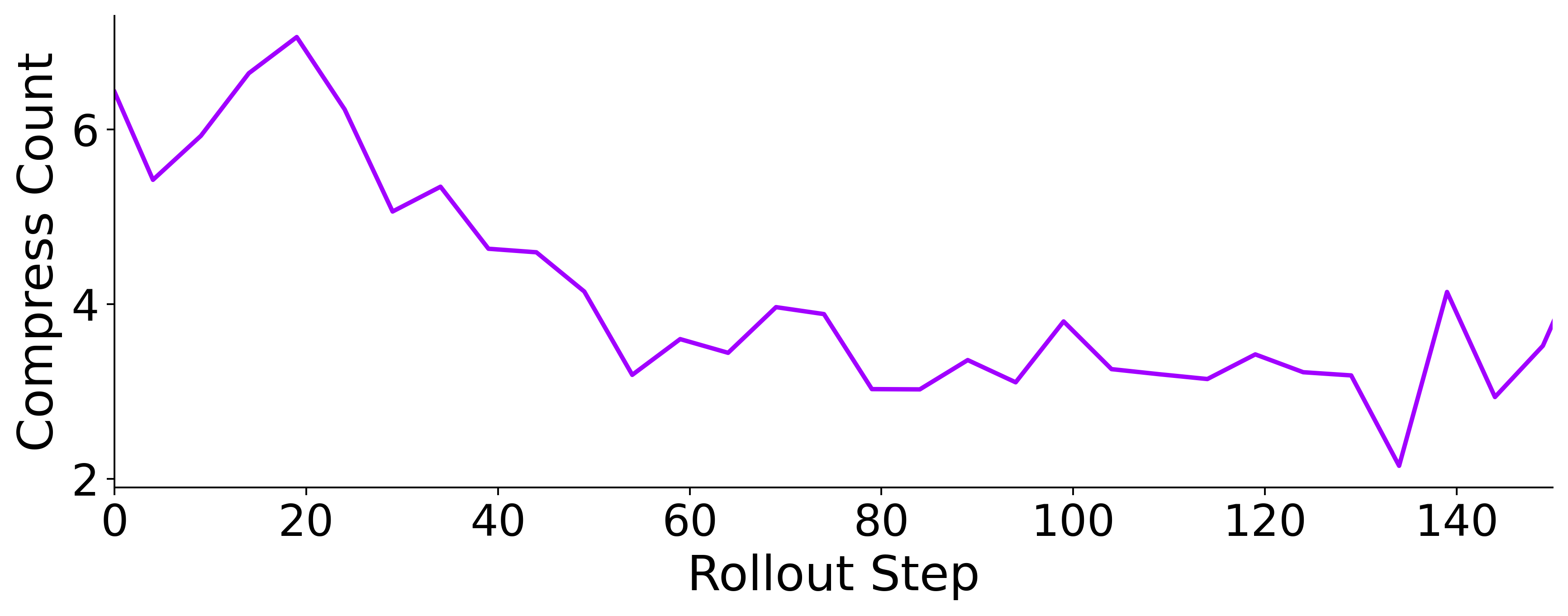}
        \caption{Mean \textsc{CompressExperience} calls per episode.}
        \label{fig:eval_compress}
    \end{subfigure}
    \hfill
    \begin{subfigure}[b]{0.48\textwidth}
        \centering
        \includegraphics[width=\textwidth]{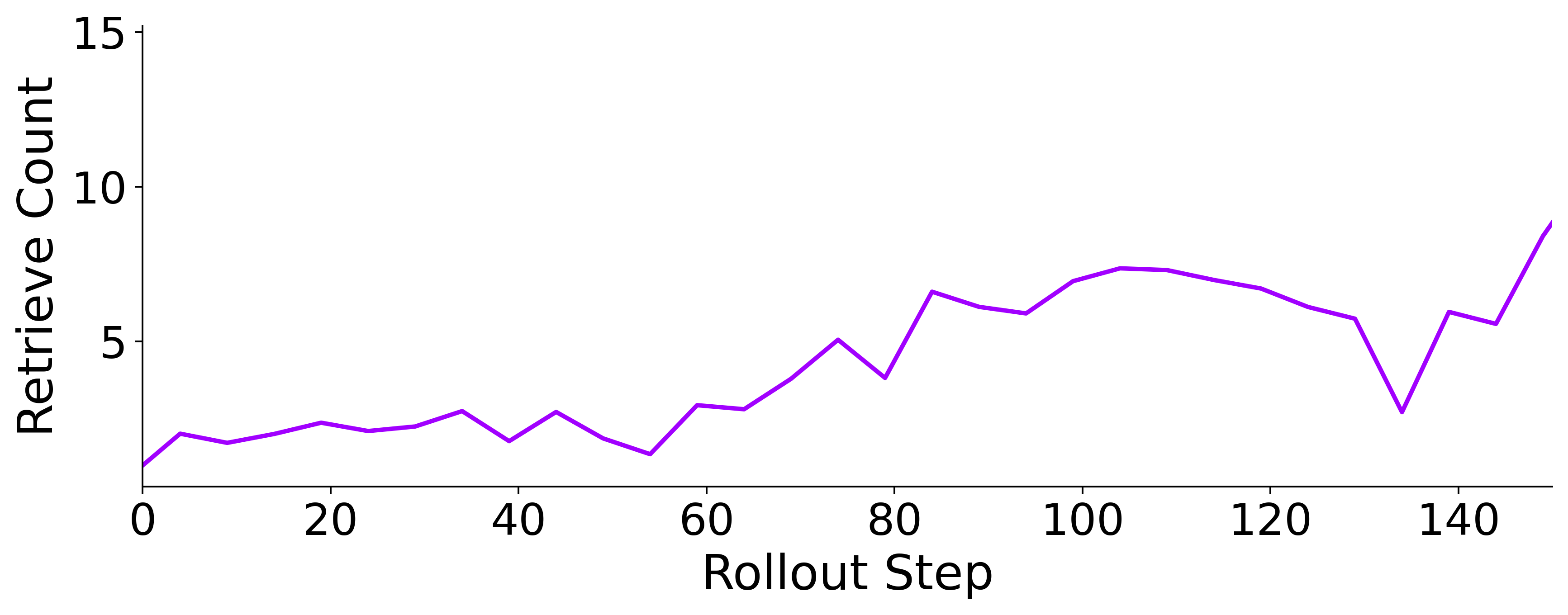}
        \caption{Mean \textsc{ReadExperience} calls per episode.}
        \label{fig:eval_retrieve}
    \end{subfigure}
    \caption{Memory tool usage on the evaluation set during training. (a) Compress count decreases from 6.5 to 3 as the agent completes tasks more efficiently. (b) Retrieve count increases from 1 to 6--7, showing that RL reinforces retrieval behavior rather than suppressing it.}
    \label{fig:memory_usage}
    \vspace{-0.1cm}
\end{figure}

%% file: contents/conclusion.tex
\section{Conclusion}
\label{sec:conclusion}

We presented \sys, a long-horizon tool-using agent that operates under fixed context windows by separating a
compact in-context \emph{indexed summary} from full-fidelity artifacts stored in an external experience database under
stable indices. This \emph{Indexed Experience Memory} turns long trajectories into a pointer-heavy workflow, enabling
the agent to keep actionable state in context while dereferencing exact evidence only when needed. We further
introduced \sysrl, which learns both compression and retrieval behaviors with reward shaping tailored to
indexed memory usage and segmented trajectory processing for episodes with multiple compressions. Across challenging
long-horizon tasks, \sys improves task success while
using a significant smaller working context. These results suggest that learning
how to summarize, index, and dereference experience is a complementary scaling axis for building more persistent and
reliable LLM agents.

%% file: contents/appendix.tex
\appendix

\section{Additional Proofs for Theoretical Analysis}
\label{sec:theory_proofs}

In this appendix, we provide full proofs for the two theoretical results in \autoref{sec:theory}. These results formalize the two central properties of \sys emphasized in the main paper: preserving decision quality through bounded explicit dereferencing, and keeping working context bounded even as the full message history grows.

\subsection{Proof of Optimality Equivalence with Bounded Dereferences (Proposition 1)}

\begin{proof}
Fix an arbitrary step \(t\) and an arbitrary full message history \(M_t^{\mathrm{full}}\). Let \((\sigma_t, D_t)\) be the \sys state induced by \(M_t^{\mathrm{full}}\). By the \(B\)-bounded decision-sufficiency assumption, there exists a selector \(g\) such that the set of indices \(I_t \triangleq g(\sigma_t)\) satisfies \(|I_t| \leq B\). Let the corresponding dereferenced content be denoted by \(R_t \triangleq \{D_t[i]\}_{i \in I_t}\).

Again by assumption, the optimal full-context policy satisfies \(\pi^*(\cdot \mid M_t^{\mathrm{full}}) = \mu(\cdot \mid \sigma_t, R_t)\). We now define a \sys policy \(\pi_{\mathrm{IEM}}\) by first selecting the indices \(I_t = g(\sigma_t)\), then reading the associated archived blocks from \(D_t\), and finally applying the decision function \(\mu\) to the pair \((\sigma_t, R_t)\). Formally, for every \sys state \((\sigma_t, D_t)\), define \(\pi_{\mathrm{IEM}}(\cdot \mid \sigma_t, D_t) \triangleq \mu\!\left(\cdot \mid \sigma_t, \{D_t[i]\}_{i \in g(\sigma_t)}\right)\).

By construction, for every message history \(M_t^{\mathrm{full}}\), the action distribution induced by \(\pi_{\mathrm{IEM}}\) is exactly the same as that induced by \(\pi^*\). In other words, \(\pi_{\mathrm{IEM}}(\cdot \mid \sigma_t, D_t) = \pi^*(\cdot \mid M_t^{\mathrm{full}})\) for all reachable states. Therefore, both policies induce the same distribution over actions at every step, and hence the same distribution over trajectories under the same environment dynamics. It follows that they achieve the same expected return, namely \(J(\pi_{\mathrm{IEM}}) = J(\pi^*)\).

Finally, the number of dereferences used by \(\pi_{\mathrm{IEM}}\) at each step is at most \(|g(\sigma_t)| \leq B\) by assumption. Therefore, \(\pi_{\mathrm{IEM}}\) is a valid \sys policy with at most \(B\) calls to \textsc{ReadExperience}\((\cdot)\) per step, and it matches the performance of the optimal full-context policy.
\end{proof}

\subsection{Proof of Bounded Working Context (Proposition 2)}

\begin{proof}
By definition, the working-context length at step \(t\) is \(C_t^{\mathrm{work}} = |\sigma_t| + \sum_{i \in I_t} |D_t[i]|\). We bound each term separately.

By assumption, the indexed summary length is bounded as \(|\sigma_t| \leq \tau_\sigma\). Also by assumption, the number of dereferenced blocks is bounded by \(|I_t| \leq B\), and each retrieved block has length at most \(L\). Therefore, \(\sum_{i \in I_t} |D_t[i]| \leq |I_t| \cdot L \leq BL\).

Substituting these bounds into the definition of \(C_t^{\mathrm{work}}\), we obtain \(C_t^{\mathrm{work}} \leq \tau_\sigma + BL\). Define the right-hand side as \(C_{\mathrm{work}}^{\max}\). This proves that the working-context length is uniformly bounded for all \(t\).

Next, by definition, the compression ratio is \(\rho_t = \frac{C_t^{\mathrm{full}}}{C_t^{\mathrm{work}}}\), where \(C_t^{\mathrm{full}} = |M_t^{\mathrm{full}}|\) is the full message history length. Since \(C_t^{\mathrm{work}} \leq C_{\mathrm{work}}^{\max}\), we have \(\rho_t = \frac{C_t^{\mathrm{full}}}{C_t^{\mathrm{work}}} \geq \frac{C_t^{\mathrm{full}}}{C_{\mathrm{work}}^{\max}}\). Thus, the compression ratio grows at least linearly with the full message history length.

In particular, as the message history becomes longer and \(C_t^{\mathrm{full}} \to \infty\), the denominator \(C_t^{\mathrm{work}}\) remains bounded by the constant \(C_{\mathrm{work}}^{\max}\), while the numerator \(C_t^{\mathrm{full}}\) continues to grow. Therefore, \(\rho_t \to \infty\). This shows that \sys can keep working context bounded even as the total message history grows without bound.

Since the system prompt \(m_0\) and task instruction \(u\) are fixed, bounded working context also implies bounded effective in-context computation.
\end{proof}

\section{Details of the Used Prompts}
\label{sec:prompts}

The detailed used prompts in our experiments are as follows:

\begin{tcolorbox}[colback=gray!5, colframe=gray!50, title=Full System Prompt for Modified ALFWorld with Memory Tools, fonttitle=\small\bfseries, breakable, left=2mm, right=2mm]
\scriptsize
\begin{verbatim}
You are an intelligent agent in a household environment (ALFWorld).

# Your Goal
Complete household tasks by navigating rooms, interacting with objects, and manipulating
them appropriately.

# Available Tools

1. **execute_action** - Execute an action in the environment
   - Parameter: action (string) - The action to execute in natural language
   - Returns: Observation of the result

2. **finish** - Indicate task completion
   - Parameter: success (boolean) - Whether the task was completed successfully

# Action Format
Actions are natural language commands. Common actions include:
- Navigation: "go to desk 1", "go to drawer 2", "go to fridge 1"
- Picking up: "pick up book 1", "pick up apple 1"
- Placing: "put book 1 in/on desk 1", "put apple 1 in fridge 1"
- Opening/Closing: "open drawer 1", "close fridge 1"
- Using: "use lamp 1" (turn on), "cool apple 1 with fridge 1", "heat potato 1 with
  microwave 1"
- Looking: "examine desk 1", "look"

# Tips
1. Pay attention to the task description - it tells you exactly what to do
2. Use "go to [location]" to navigate to objects
3. You must be at a location before you can interact with it
4. Some containers (drawers, fridges) need to be opened before you can see/use their
   contents
5. Be systematic: find the object, pick it up, find the destination, put it down

##############################################################################
#                         MANDATORY REQUIREMENTS                              #
##############################################################################

>>> REQUIREMENT 1: YOU MUST CALL A TOOL IN EVERY RESPONSE <<<
- Every response MUST contain a tool call
- Plain text responses WITHOUT a tool call will be REJECTED

>>> REQUIREMENT 2: USE execute_action FOR ALL INTERACTIONS <<<
- Use execute_action to perform actions in the environment
- Only use finish when you have completed the task

<IMPORTANT>
- You MUST provide consise reasoning BEFORE every action.
</IMPORTANT>

=== CRITICAL: THREE OBJECTIVES ===

You have THREE equally important goals:
1. SOLVE THE TASK correctly
2. KEEP working context UNDER the threshold (shown in [Context Status] after EVERY
   observation)
3. NEVER make redundant tool calls (same tool + same arguments without file/status
   changes in between)

SEVERE PENALTIES (can nullify your task success reward):
- Context overflow: If working > threshold, you receive a SEVERE PENALTY that can
  completely offset solving the task
- Redundant tool calls: Calling the SAME tool with IDENTICAL arguments twice (without
  modifying files/status in between) results in a SEVERE PENALTY
- These penalties are AS IMPORTANT as solving the task - poor memory management can
  make a solved task worth ZERO

MANDATORY PRACTICES:
- Monitor [Context Status: working tokens=X, threshold=Z] after EVERY observation
- Compress BEFORE working exceeds threshold (don't wait until it's too late!)
- When compressing, store BROAD coverage in db_blocks - include everything you might
  need later

CRITICAL - After compression:
- Compressed messages are DELETED from context. The ONLY way to access them is
  ReadExperience.
- If you need past information, you MUST call ReadExperience(db_index) - re-running
  tools without file/status changes is forbidden and penalized.

-- BEGIN FUNCTION: CompressExperience --
Description:
Compress working context to database for later retrieval. Replaces all messages
(except system prompt and task description) with your summary.

Usage:
  * Check [Context Status: working tokens=X, threshold=Z] at the end of each
    observation
  * Strongly recommended when working > 0.8 * threshold
  * Exceeding threshold will result in penalty
  * After compression, use ReadExperience to get saved content instead of re-running
    tools
  * When compressing multiple times: include ALL previous indices in your new summary
    (copy them over), then add new ones

Parameters:
  1. summary (string, required)
     Index map listing ALL stored indices (both old and new). Format:
     - <db_index> - <what it contains>
     - <db_index> - <what it contains>
     Include current status and next steps at the end.

  2. db_blocks (array, required)
     List of content blocks to store. Two options:

     Option A - Write content yourself:
       * db_index (string): Unique key, e.g. "ctx_code_001"
       * db_content (string): Content you write/summarize

     Option B - System auto-extracts from current conversation:
       The system finds text between your anchors and saves it automatically.
       * db_index (string): Unique key
       * start_anchor (string): REQUIRED - exact text where extraction STARTS
       * mid_anchor (string): REQUIRED - exact text that MUST appear in the middle
       * end_anchor (string): REQUIRED - exact text where extraction ENDS
       ALL THREE anchors are REQUIRED. Missing any anchor = failure.

       IMPORTANT for anchors:
       - Choose your own anchors that uniquely identify the content boundaries
       - start_anchor: unique text at the START of what you want to extract
       - mid_anchor: unique text somewhere in the MIDDLE (for verification)
       - end_anchor: unique text at the END of what you want to extract
       - Keep anchors SHORT (20-100 chars), NOT entire code blocks
       - Good: "def _check_required", "raise ValueError", "return result"
       - Bad: copying 10+ lines of code (whitespace errors cause failures)

Tip: Use Option A for summaries. Use Option B only for large verbatim outputs (test
results, stack traces) where you want exact copy.

Example:
<tool_call>
{"name": "CompressExperience", "arguments": {"summary": "Index map:\n- ctx_data_001
- Brief description of what's stored\n- ctx_data_002 - Brief description of what's
stored\nStatus: Current progress and next steps", "db_blocks": [{"db_index":
"ctx_data_001", "db_content": "Precise details, exact IDs, full content..."},
{"db_index": "ctx_data_002", "db_content": "More precise details..."}]}}
</tool_call>

IMPORTANT:
- summary: Keep descriptions SHORT (what type of data, not the data itself)
- db_blocks: Store PRECISE details you'll need later (exact IDs, full content,
  specific values)
- After compression, use ReadExperience(db_index) to retrieve precise details

-- END FUNCTION --

-- BEGIN FUNCTION: ReadExperience --
Description:
Retrieve previously compressed content by index.

Usage:
  * Use when you need exact details stored during compression
  * Available indices shown in [Context Status] and your summary's index map
  * Always retrieve instead of re-running tools for same information

Parameters:
  1. db_index (string, required)
     The index to retrieve. Must match exactly.

Example:
<tool_call>
{"name": "ReadExperience", "arguments": {"db_index": "ctx_code_001"}}
</tool_call>

-- END FUNCTION --

##############################################################################
#                    MEMORY MANAGEMENT FOR ALFWORLD                          #
##############################################################################

CRITICAL: How to use CompressExperience effectively in ALFWorld:

1. **summary** must contain ONLY short descriptions and index map. NEVER put raw IDs
   in summary!
   - BAD: "ctx_cabinet_001 - cabinet_bar__minus_00_dot_36_bar__plus_00_dot_38..."
   - BAD: "Locations: countertop_bar__minus_00_dot_28_bar__plus_00_dot_79..."
   - GOOD: "ctx_locations - All 20 location IDs\nctx_progress - Task progress"
   Summary is truncated to save space. Any IDs in summary WILL BE LOST.

2. **db_blocks** is the ONLY safe place to store exact IDs:
   - Store ALL location IDs in db_blocks (e.g., db_index="ctx_locations")
   - Store object IDs you've found (e.g., db_index="ctx_objects")
   - These IDs are required to interact with objects in future actions

3. **After compression**, you MUST call ReadExperience(db_index) to get IDs back:
   - Summary does NOT contain IDs (they are truncated)
   - The ONLY way to get exact IDs is ReadExperience

3. **After compression**, call ReadExperience(db_index) to retrieve precise IDs
   - Don't try to remember IDs from memory - they are deleted after compression
   - Always retrieve before taking actions that need specific location names

ALFWORLD-SPECIFIC EXAMPLE:

Step 1 - Compress (store IDs in db_blocks, NOT in summary):
<tool_call>
{"name": "CompressExperience", "arguments": {
  "summary": "Index map:\n- ctx_locations - All room location IDs\n- ctx_progress
  - Task status\nStatus: Found butterknife, need to clean and place on table",
  "db_blocks": [
    {"db_index": "ctx_locations", "db_content": "countertop_bar__minus_00_dot_28_bar
    __plus_00_dot_79_bar__plus_01_dot_93\ndrawer_bar__minus_00_dot_33_bar__plus_00
    _dot_32_bar__plus_01_dot_72\nsinkbasin_bar__plus_01_dot_13_bar__plus_00_dot_00
    _bar__minus_01_dot_33\ndiningtable_bar__plus_01_dot_02_bar__plus_00_dot_00_bar
    __plus_01_dot_61"},
    {"db_index": "ctx_progress", "db_content": "Task: put clean butterknife on
    diningtable\nFound: butterknife at countertop\nInventory: butterknife_bar__minus
    _00_dot_77_bar__plus_00_dot_90_bar__minus_01_dot_68\nNext: go to sinkbasin to
    clean, then to diningtable"}
  ]
}}
</tool_call>

Step 2 - After compression, retrieve IDs before navigating:
<tool_call>
{"name": "ReadExperience", "arguments": {"db_index": "ctx_locations"}}
</tool_call>
-> Returns: "countertop_bar__...\ndrawer_bar__...\nsinkbasin_bar__...\n
   diningtable_bar__..."
Now you can use these IDs: execute_action("go to sinkbasin_bar__plus_01_dot_13...")

This returns the exact cabinet IDs so you don't re-check the same locations.

# Tool Call Format
Use the following JSON format inside <tool_call> tags:

<tool_call>
{"name": "tool_name", "arguments": {"param1": "value1"}}
</tool_call>

Examples:
<tool_call>
{"name": "execute_action", "arguments": {"action": "go to desk 1"}}
</tool_call>

<tool_call>
{"name": "execute_action", "arguments": {"action": "pick up book 1"}}
</tool_call>

<tool_call>
{"name": "finish", "arguments": {"success": true}}
</tool_call>
\end{verbatim}
\end{tcolorbox}

%% file: reference.bib
@inproceedings{yao2023react,
  title={React: Synergizing reasoning and acting in language models},
  author={Yao, Shunyu and Zhao, Jeffrey and Yu, Dian and Du, Nan and Shafran, Izhak and Narasimhan, Karthik and Cao, Yuan},
  booktitle={International Conference on Learning Representations (ICLR)},
  year={2023}
}

@article{yang2025qwen3,
  title={Qwen3 technical report},
  author={Yang, An and Li, Anfeng and Yang, Baosong and Zhang, Beichen and Hui, Binyuan and Zheng, Bo and Yu, Bowen and Gao, Chang and Huang, Chengen and Lv, Chenxu and others},
  journal={arXiv preprint arXiv:2505.09388},
  year={2025}
}

@article{chen2025minimax,
  title={MiniMax-M1: Scaling Test-Time Compute Efficiently with Lightning Attention},
  author={Chen, Aili and Li, Aonian and Gong, Bangwei and Jiang, Binyang and Fei, Bo and Yang, Bo and Shan, Boji and Yu, Changqing and Wang, Chao and Zhu, Cheng and others},
  journal={arXiv preprint arXiv:2506.13585},
  year={2025}
}

@misc{anthropic2025claude4,
  title   = {Introducing Claude 4},
  author  = {Anthropic},
  year    = {2025},
  url     = {https://www.anthropic.com/news/claude-4
}
}

@article{zeng2026glm,
  title={GLM-5: from Vibe Coding to Agentic Engineering},
  author={Zeng, Aohan and Lv, Xin and Hou, Zhenyu and Du, Zhengxiao and Zheng, Qinkai and Chen, Bin and Yin, Da and Ge, Chendi and Xie, Chengxing and Wang, Cunxiang and others},
  journal={arXiv preprint arXiv:2602.15763},
  year={2026}
}

@article{team2026kimi,
  title={Kimi K2. 5: Visual Agentic Intelligence},
  author={Team, Kimi and Bai, Tongtong and Bai, Yifan and Bao, Yiping and Cai, SH and Cao, Yuan and Charles, Y and Che, HS and Chen, Cheng and Chen, Guanduo and others},
  journal={arXiv preprint arXiv:2602.02276},
  year={2026}
}

@misc{openai2025gpt5,
  title   = {Introducing gpt-5},
  author  = {OpenAI},
  year    = {2025},
  url     = {https://openai.com/index/introducing-gpt-5/
}
}

@misc{google2025Gemini3ProModelCard,
  title   = {Gemini-3-Pro-Model-Card},
  author  = {Google DeepMind},
  year    = {2025},
  url     = {https://storage.googleapis.com/deepmind-media/Model-Cards/Gemini-3-Pro-Model-Card.pdf
}
}

@article{liu2026klong,
  title={KLong: Training LLM Agent for Extremely Long-horizon Tasks},
  author={Liu, Yue and Hu, Zhiyuan and Sung, Flood and Zhang, Jiaheng and Hooi, Bryan},
  journal={arXiv preprint arXiv:2602.17547},
  year={2026}
}

@article{lu2025scaling,
  title={Scaling llm multi-turn rl with end-to-end summarization-based context management},
  author={Lu, Miao and Sun, Weiwei and Du, Weihua and Ling, Zhan and Yao, Xuesong and Liu, Kang and Chen, Jiecao},
  journal={arXiv preprint arXiv:2510.06727},
  year={2025}
}

@article{teyler1986hippocampal,
  title={The hippocampal memory indexing theory.},
  author={Teyler, Timothy J and DiScenna, Pascal},
  journal={Behavioral neuroscience},
  volume={100},
  number={2},
  pages={147},
  year={1986},
  publisher={American Psychological Association}
}

@article{andy1998extended,
  title={The extended mind},
  author={Andy, Clark and David, Chalmers},
  journal={Analysis},
  volume={58},
  number={1},
  pages={7--19},
  year={1998}
}

@inproceedings{Zhong-AAAI2024,
  title={Memorybank: Enhancing large language models with long-term memory},
  author={Zhong, Wanjun and Guo, Lianghong and Gao, Qiqi and Ye, He and Wang, Yanlin},
  booktitle={Proceedings of the AAAI Conference on Artificial Intelligence},
  volume={38},
  number={17},
  pages={19724--19731},
  year={2024}
}

@article{Packer-Arxiv2023,
  title={MemGPT: Towards LLMs as Operating Systems.},
  author={Packer, Charles and Fang, Vivian and Patil, Shishir\_G and Lin, Kevin and Wooders, Sarah and Gonzalez, Joseph\_E},
  year={2023},
  publisher={ArXiv}
}

@article{Rezazadeh-ICLR2025,
  title={From isolated conversations to hierarchical schemas: Dynamic tree memory representation for llms},
  author={Rezazadeh, Alireza and Li, Zichao and Wei, Wei and Bao, Yujia},
  journal={arXiv preprint arXiv:2410.14052},
  year={2024}
}

@article{Xu-NIPS2025,
  title={A-mem: Agentic memory for llm agents},
  author={Xu, Wujiang and Liang, Zujie and Mei, Kai and Gao, Hang and Tan, Juntao and Zhang, Yongfeng},
  journal={arXiv preprint arXiv:2502.12110},
  year={2025}
}

@article{Hu-Survey2025,
  title={Memory in the Age of AI Agents},
  author={Hu, Yuyang and Liu, Shichun and Yue, Yanwei and Zhang, Guibin and Liu, Boyang and Zhu, Fangyi and Lin, Jiahang and Guo, Honglin and Dou, Shihan and Xi, Zhiheng and others},
  journal={arXiv preprint arXiv:2512.13564},
  year={2025}
}

@article{Chhikara-Arxiv2025,
  title={Mem0: Building production-ready ai agents with scalable long-term memory},
  author={Chhikara, Prateek and Khant, Dev and Aryan, Saket and Singh, Taranjeet and Yadav, Deshraj},
  journal={arXiv preprint arXiv:2504.19413},
  year={2025}
}

@article{Kang-Arxiv2025,
  title={Memory OS of AI Agent},
  author={Kang, Jiazheng and Ji, Mingming and Zhao, Zhe and Bai, Ting},
  journal={arXiv preprint arXiv:2506.06326},
  year={2025}
}

@inproceedings{Tan-ACL2025,
  title={In prospect and retrospect: Reflective memory management for long-term personalized dialogue agents},
  author={Tan, Zhen and Yan, Jun and Hsu, I-Hung and Han, Rujun and Wang, Zifeng and Le, Long and Song, Yiwen and Chen, Yanfei and Palangi, Hamid and Lee, George and others},
  booktitle={Proceedings of the 63rd Annual Meeting of the Association for Computational Linguistics (Volume 1: Long Papers)},
  pages={8416--8439},
  year={2025}
}

@article{Rasmussen-Arxiv2025,
  title={Zep: a temporal knowledge graph architecture for agent memory},
  author={Rasmussen, Preston and Paliychuk, Pavlo and Beauvais, Travis and Ryan, Jack and Chalef, Daniel},
  journal={arXiv preprint arXiv:2501.13956},
  year={2025}
}

@article{Zhou-Arxiv2023,
  title={Recurrentgpt: Interactive generation of (arbitrarily) long text},
  author={Zhou, Wangchunshu and Jiang, Yuchen Eleanor and Cui, Peng and Wang, Tiannan and Xiao, Zhenxin and Hou, Yifan and Cotterell, Ryan and Sachan, Mrinmaya},
  journal={arXiv preprint arXiv:2305.13304},
  year={2023}
}

@article{Long-Arxiv2025,
  title={Seeing, listening, remembering, and reasoning: A multimodal agent with long-term memory},
  author={Long, Lin and He, Yichen and Ye, Wentao and Pan, Yiyuan and Lin, Yuan and Li, Hang and Zhao, Junbo and Li, Wei},
  journal={arXiv preprint arXiv:2508.09736},
  year={2025}
}

@article{Liu-Arxiv2023,
  title={Think-in-memory: Recalling and post-thinking enable llms with long-term memory},
  author={Liu, Lei and Yang, Xiaoyan and Shen, Yue and Hu, Binbin and Zhang, Zhiqiang and Gu, Jinjie and Zhang, Guannan},
  journal={arXiv preprint arXiv:2311.08719},
  year={2023}
}

@inproceedings{Chen-Arxiv2025,
  title={Compress to impress: Unleashing the potential of compressive memory in real-world long-term conversations},
  author={Chen, Nuo and Li, Hongguang and Chang, Jianhui and Huang, Juhua and Wang, Baoyuan and Li, Jia},
  booktitle={Proceedings of the 31st International Conference on Computational Linguistics},
  pages={755--773},
  year={2025}
}

@inproceedings{Wang-NAACL2024,
  title={Recmind: Large language model powered agent for recommendation},
  author={Wang, Yancheng and Jiang, Ziyan and Chen, Zheng and Yang, Fan and Zhou, Yingxue and Cho, Eunah and Fan, Xing and Lu, Yanbin and Huang, Xiaojiang and Yang, Yingzhen},
  booktitle={Findings of the Association for Computational Linguistics: NAACL 2024},
  pages={4351--4364},
  year={2024}
}

@article{Jimenez-NIPS2024,
  title={Hipporag: Neurobiologically inspired long-term memory for large language models},
  author={Jimenez Gutierrez, Bernal and Shu, Yiheng and Gu, Yu and Yasunaga, Michihiro and Su, Yu},
  journal={Advances in Neural Information Processing Systems},
  volume={37},
  pages={59532--59569},
  year={2024}
}

@inproceedings{Zhao-AAAI2024,
  title={Expel: Llm agents are experiential learners},
  author={Zhao, Andrew and Huang, Daniel and Xu, Quentin and Lin, Matthieu and Liu, Yong-Jin and Huang, Gao},
  booktitle={Proceedings of the AAAI Conference on Artificial Intelligence},
  volume={38},
  number={17},
  pages={19632--19642},
  year={2024}
}

@article{Zhou-Arxiv2025,
  title={Memento: Fine-tuning llm agents without fine-tuning llms},
  author={Zhou, Huichi and Chen, Yihang and Guo, Siyuan and Yan, Xue and Lee, Kin Hei and Wang, Zihan and Lee, Ka Yiu and Zhang, Guchun and Shao, Kun and Yang, Linyi and others},
  journal={arXiv preprint arXiv:2508.16153},
  year={2025}
}

@article{zhang2025memgen,
  title={Memgen: Weaving generative latent memory for self-evolving agents},
  author={Zhang, Guibin and Fu, Muxin and Yan, Shuicheng},
  journal={arXiv preprint arXiv:2509.24704},
  year={2025}
}

@article{Wang-AWM-ICML2024,
  title={Agent workflow memory},
  author={Wang, Zora Zhiruo and Mao, Jiayuan and Fried, Daniel and Neubig, Graham},
  journal={arXiv preprint arXiv:2409.07429},
  year={2024}
}

@article{Shinn-NIPS2023,
  title={Reflexion: Language agents with verbal reinforcement learning},
  author={Shinn, Noah and Cassano, Federico and Gopinath, Ashwin and Narasimhan, Karthik and Yao, Shunyu},
  journal={Advances in Neural Information Processing Systems},
  volume={36},
  pages={8634--8652},
  year={2023}
}

@article{Ouyang-Arxiv2025,
  title={Reasoningbank: Scaling agent self-evolving with reasoning memory},
  author={Ouyang, Siru and Yan, Jun and Hsu, I and Chen, Yanfei and Jiang, Ke and Wang, Zifeng and Han, Rujun and Le, Long T and Daruki, Samira and Tang, Xiangru and others},
  journal={arXiv preprint arXiv:2509.25140},
  year={2025}
}

@inproceedings{corley2005measuring,
  title={Measuring the semantic similarity of texts},
  author={Corley, Courtney D and Mihalcea, Rada},
  booktitle={Proceedings of the ACL workshop on empirical modeling of semantic equivalence and entailment},
  pages={13--18},
  year={2005}
}

@article{Wang-m+-ICML2025,
  title={M+: Extending MemoryLLM with Scalable Long-Term Memory},
  author={Wang, Yu and Krotov, Dmitry and Hu, Yuanzhe and Gao, Yifan and Zhou, Wangchunshu and McAuley, Julian and Gutfreund, Dan and Feris, Rogerio and He, Zexue},
  journal={arXiv preprint arXiv:2502.00592},
  year={2025}
}

@inproceedings{Qian-memorag-WC2025,
  title={Memorag: Boosting long context processing with global memory-enhanced retrieval augmentation},
  author={Qian, Hongjin and Liu, Zheng and Zhang, Peitian and Mao, Kelong and Lian, Defu and Dou, Zhicheng and Huang, Tiejun},
  booktitle={Proceedings of the ACM on Web Conference 2025},
  pages={2366--2377},
  year={2025}
}

@inproceedings{kwon2023efficient,
  title={Efficient memory management for large language model serving with pagedattention},
  author={Kwon, Woosuk and Li, Zhuohan and Zhuang, Siyuan and Sheng, Ying and Zheng, Lianmin and Yu, Cody Hao and Gonzalez, Joseph and Zhang, Hao and Stoica, Ion},
  booktitle={Proceedings of the 29th symposium on operating systems principles},
  pages={611--626},
  year={2023}
}

@article{zheng-sglang-NIPS2024,
  title={Sglang: Efficient execution of structured language model programs},
  author={Zheng, Lianmin and Yin, Liangsheng and Xie, Zhiqiang and Sun, Chuyue Livia and Huang, Jeff and Yu, Cody Hao and Cao, Shiyi and Kozyrakis, Christos and Stoica, Ion and Gonzalez, Joseph E and others},
  journal={Advances in neural information processing systems},
  volume={37},
  pages={62557--62583},
  year={2024}
}

@article{An-ICLR2025,
  title={Why does the effective context length of LLMs fall short?},
  author={An, Chenxin and Zhang, Jun and Zhong, Ming and Li, Lei and Gong, Shansan and Luo, Yao and Xu, Jingjing and Kong, Lingpeng},
  journal={arXiv preprint arXiv:2410.18745},
  year={2024}
}

@article{wu-longmemeval-ICLR2025,
  title={Longmemeval: Benchmarking chat assistants on long-term interactive memory},
  author={Wu, Di and Wang, Hongwei and Yu, Wenhao and Zhang, Yuwei and Chang, Kai-Wei and Yu, Dong},
  journal={arXiv preprint arXiv:2410.10813},
  year={2024}
}

@article{zhou2025mem1,
  title={MEM1: Learning to Synergize Memory and Reasoning for Efficient Long-Horizon Agents},
  author={Zhou, Zijian and Qu, Ao and Wu, Zhaoxuan and Kim, Sunghwan and Prakash, Alok and Rus, Daniela and Zhao, Jinhua and Low, Bryan Kian Hsiang and Liang, Paul Pu},
  journal={arXiv preprint arXiv:2506.15841},
  year={2025}
}

@article{yu2025memagent,
  title={MemAgent: Reshaping Long-Context LLM with Multi-Conv RL-based Memory Agent},
  author={Yu, Hongli and Chen, Tinghong and Feng, Jiangtao and Chen, Jiangjie and Dai, Weinan and Yu, Qiying and Zhang, Ya-Qin and Ma, Wei-Ying and Liu, Jingjing and Wang, Mingxuan and others},
  journal={arXiv preprint arXiv:2507.02259},
  year={2025}
}

@article{yan2025memory,
  title={Memory-r1: Enhancing large language model agents to manage and utilize memories via reinforcement learning},
  author={Yan, Sikuan and Yang, Xiufeng and Huang, Zuchao and Nie, Ercong and Ding, Zifeng and Li, Zonggen and Ma, Xiaowen and Kersting, Kristian and Pan, Jeff Z and Sch{\"u}tze, Hinrich and others},
  journal={arXiv preprint arXiv:2508.19828},
  year={2025}
}

@article{wang2025mem,
  title={Mem-$\{$$\backslash$alpha$\}$: Learning memory construction via reinforcement learning},
  author={Wang, Yu and Takanobu, Ryuichi and Liang, Zhiqi and Mao, Yuzhen and Hu, Yuanzhe and McAuley, Julian and Wu, Xiaojian},
  journal={arXiv preprint arXiv:2509.25911},
  year={2025}
}

@article{wu2025resum,
  title={ReSum: Unlocking Long-Horizon Search Intelligence via Context Summarization},
  author={Wu, Xixi and Li, Kuan and Zhao, Yida and Zhang, Liwen and Ou, Litu and Yin, Huifeng and Zhang, Zhongwang and Yu, Xinmiao and Zhang, Dingchu and Jiang, Yong and others},
  journal={arXiv preprint arXiv:2509.13313},
  year={2025}
}

@article{sun2025scaling,
  title={Scaling long-horizon llm agent via context-folding},
  author={Sun, Weiwei and Lu, Miao and Ling, Zhan and Liu, Kang and Yao, Xuesong and Yang, Yiming and Chen, Jiecao},
  journal={arXiv preprint arXiv:2510.11967},
  year={2025}
}

@article{ye2025agentfold,
  title={AgentFold: Long-Horizon Web Agents with Proactive Context Management},
  author={Ye, Rui and Zhang, Zhongwang and Li, Kuan and Yin, Huifeng and Tao, Zhengwei and Zhao, Yida and Su, Liangcai and Zhang, Liwen and Qiao, Zile and Wang, Xinyu and others},
  journal={arXiv preprint arXiv:2510.24699},
  year={2025}
}

@misc{slime_github,
  author       = {Zilin Zhu and Chengxing Xie and Xin Lv and slime Contributors},
  title        = {slime: An LLM post-training framework for RL Scaling},
  year         = {2025},
  howpublished = {\url{https://github.com/THUDM/slime}},
  note         = {GitHub repository. Corresponding author: Xin Lv},
  urldate      = {2025-06-19}
}

@article{shridhar2020alfworld,
  title={Alfworld: Aligning text and embodied environments for interactive learning},
  author={Shridhar, Mohit and Yuan, Xingdi and C{\^o}t{\'e}, Marc-Alexandre and Bisk, Yonatan and Trischler, Adam and Hausknecht, Matthew},
  journal={arXiv preprint arXiv:2010.03768},
  year={2020}
}
